\documentclass[journal]{IEEEtran}

\IEEEoverridecommandlockouts                              % This command is only needed if 
\title{
Enhanced Mean Field Game for Interactive Decision-Making with Varied Stylish Multi-Vehicles}

\author{Liancheng Zheng$^{1}$, 
        Zhen Tian$^{2,*}$, 
        Yangfan He$^{3}$, 
        Shuo Liu$^{4}$, 
        Huilin Chen$^{5}$, Fujiang Yuan$^{6,*}$ and Yanhong Peng$^{7,*}$%

\thanks{$^{1}$
School of Mechanical Engineering, Shandong Huayu University of Technology, Dezhou 253034, China}%
\thanks{$^{2}$James Watt School of Engineering, University of Glasgow, Glasgow G12 8QQ, United Kingdom.}%
\thanks{$^{3}$Computer science with the University of Minnesota - Twin Cities, Minneapolis, MN, USA.}%
\thanks{$^{4}$Boston University, Brookline, MA, USA}
\thanks{$^{5}$Faculty of Computer Science and Information Technology, University of Malaya, Kuala Lumpur, Malaysia}%
\thanks{$^{6}$School of Computer Science and Technology, Taiyuan Normal University, Jinzhong, 030619, Shanxi, China}%
\thanks{$^{7}$College of Mechanical Engineering, Chongqing University of Technology, Chongqing, 400054, China.}
\thanks{*Corresponding author.}%
}
\usepackage{xcolor}

\usepackage{hyperref}
\usepackage{cite}
\usepackage{algorithm}  
\usepackage{algpseudocode}
\usepackage{multirow}
\usepackage{amsmath}
\usepackage{amssymb}
\usepackage{epsfig} % for postscript graphics files

\usepackage{amsthm}

\newtheorem{theorem}{Theorem}[section]
\newtheorem{lemma}[theorem]{Lemma}

\usepackage{graphicx}
\usepackage{caption}
\usepackage{subcaption}%小表
\usepackage{enumitem} %自定义item

\usepackage[flushleft]{threeparttable}%加表格描述用的

\usepackage{tabularx} % 使用tabularx包来灵活控制表格宽度

\usepackage{titlesec}

\begin{document}

\maketitle

%%%%%%%%%%%%%%%%%%%%%%%%%%%%%%%%%%%%%%%%%%%%%%%%%%%%%%%%%%%%%%%%%%%%%%%%%%%%%%%%
\begin{abstract}
This paper presents a MFG–based decision-making framework for autonomous driving in heterogeneous traffic. To model diverse human behaviors, we propose a quantitative driving style representation that maps abstract traits to parameters such as speed, safety factors, and reaction time. These are embedded into the MFG via a spatial influence field model. To ensure safe operation in dense traffic, we introduce a safety-critical lane-changing algorithm using dynamic safety margins, time-to-collision analysis, and multi-layered constraints. Real-world NGSIM data is used for style calibration and empirical validation. Experimental results show zero collisions across six style combinations, two 15-vehicle scenarios, and NGSIM-based trials, outperforming conventional game-theoretic baselines. Our approach offers scalable, interpretable, and behavior-aware planning for real-world autonomous driving applications.
\end{abstract}


\section{Introduction}

\IEEEPARstart{I}n recent years, autonomous driving has advanced rapidly, with substantial progress in perception~\cite{lin2025slam2,lin2024dpl}, planning~\cite{yuan2025bio}, and control~\cite{tsai2024autonomous}. Nevertheless, interactive driving scenarios remain a major challenge, where safety is paramount. As illustrated in Figure1, an autonomous vehicle (AV) navigating among human-driven vehicles (HDVs) must continuously assess risks and adjust its trajectory to avoid collisions. The growing presence of AVs in heterogeneous traffic environments has further increased the complexity of interactions, requiring decision-making frameworks that can handle diverse human driving behaviors \cite{slade2024human,wang2024multimodal}. In dense traffic, such behaviors range from aggressive lane-changing to conservative following~\cite{zhu2025fdnet,chen2024safety}, creating significant variability that traditional uniform modeling approaches fail to capture~\cite{zhao2024potential}.

Most existing methods assume homogeneous agent behavior, neglecting the influence of behavioral diversity on traffic flow~\cite{lin2025safety,chen2024end,lin2024conflicts}. Classical game-theoretic models, though mathematically elegant, struggle with scalability in large-scale multi-agent scenarios and cannot fully represent the nuanced variations of human drivers~\cite{tan2021risk,wu2022humanlike}. Likewise, conventional control strategies~\cite{li2025adaptive,li2025efficient} often exhibit excessive conservatism in dense environments, leading to reduced efficiency and unsafe reactive decisions. Control Barrier Functions (CBFs) provide formal safety guarantees by preventing the system from entering unsafe states~\cite{lin2024enhanced,ames2019control}, such as avoiding dynamic obstacles~\cite{9483029}. However, CBFs tend to be overly restrictive, forcing AVs to yield excessively and compromising efficiency. Striking a balance between safety and performance thus remains a difficult challenge in CBF-based approaches.

The emergence of data-driven methodologies has offered alternative pathways for modeling complex vehicular interactions~\cite{dai2024marp}. However, these approaches encounter significant limitations in interpretability and generalization across diverse traffic configurations~\cite{hassija2024interpreting}. Learning-based systems, despite their capacity to capture intricate behavioral patterns~\cite{tian2025evaluating,xu2024towards}, require extensive training datasets and often lack the theoretical foundations necessary for safety guarantees in critical scenarios~\cite{rasol2024exploring,peng2024predictive}. Furthermore, the black-box nature of these systems impedes their deployment in safety-critical applications where decision transparency is paramount.

In recent years, MFG theory has emerged as a promising framework for modeling interactive driving. Unlike traditional game-theoretic approaches, MFG approximates the aggregate behavior of a large population of vehicles using a mean field term, thereby avoiding the need to model every pairwise interaction explicitly and mitigating the scalability issues inherent in conventional game formulations. However, existing MFG-based methods for autonomous driving often suffer from several limitations: they lack simulations across diverse traffic scenarios, fail to compare against other game-theoretic baselines, and remain disconnected from control-level optimization. A notable exception is~\cite{zheng2025mean}, which combines MFG with MPC-QP to capture collective vehicle dynamics while applying a receding-horizon optimization scheme for real-time control. This integration addresses the conservatism of Stackelberg games and the computational burden of Nash games, providing a more flexible and scalable solution for interactive driving. Nevertheless, most prior MFG models overlook the heterogeneity of driving styles among surrounding vehicles, limiting their realism and applicability in real-world traffic. Furthermore, they fail to ensure safe interactions in densely populated environments with varied behavioral agents—an essential capability for real-world deployment.
\begin{figure*}[t]
    \centering
    \includegraphics[width=0.9\linewidth]{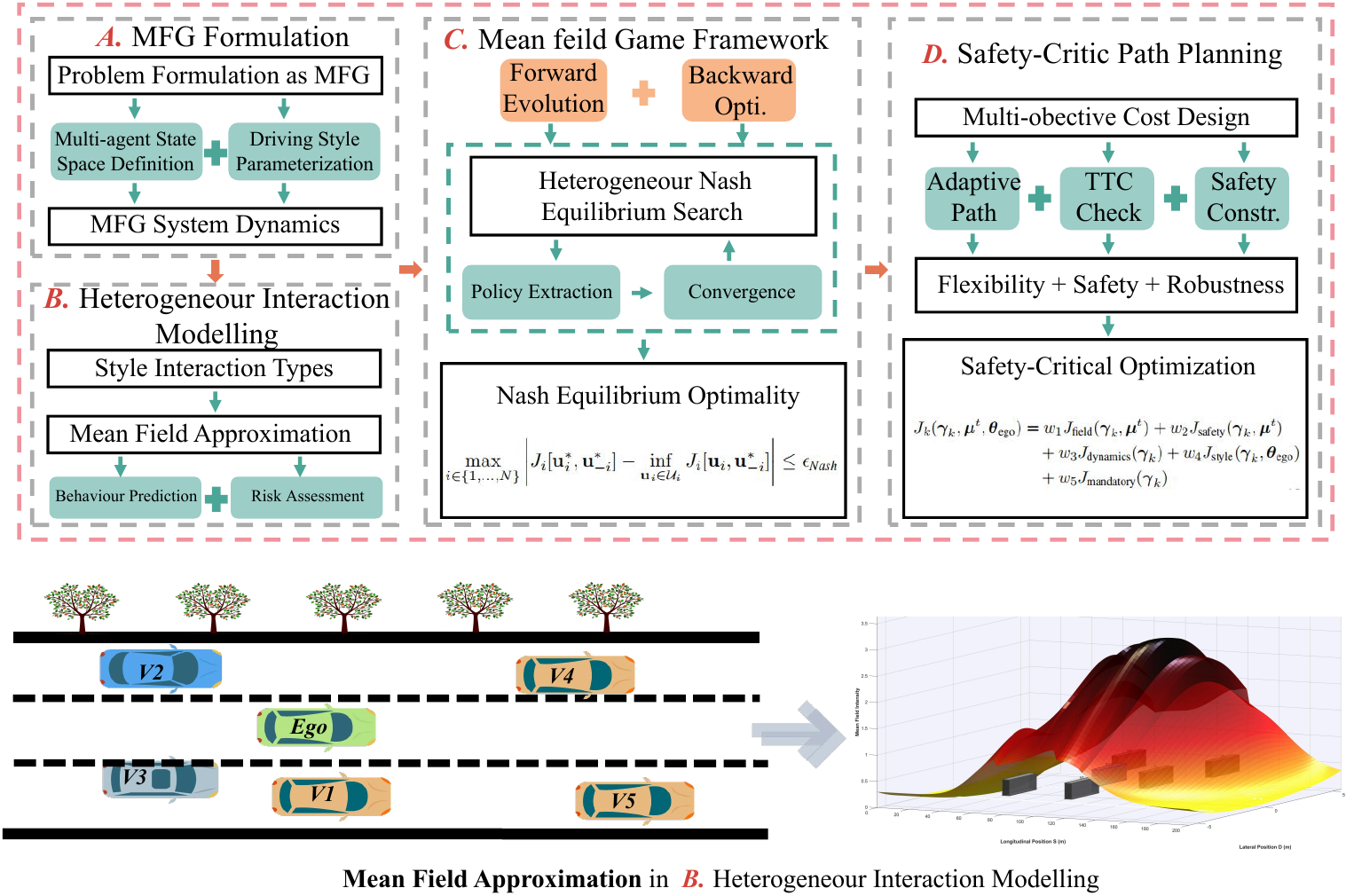}
    \caption{Framework of the heterogeneous MFG-based planner.}
    \label{fig:framework}
\end{figure*}
To overcome these limitations, we propose a comprehensive MFG-based framework that explicitly incorporates driving style heterogeneity into large-scale vehicular interaction modeling. Departing from conventional pairwise interaction schemes, our method employs mean field approximations to represent the collective influence of surrounding vehicles, each parameterized by empirically grounded behavioral characteristics. To ensure realism, we integrate trajectory data from the Next Generation Simulation (NGSIM) dataset for both calibration and validation of our driving style models. Additionally, we develop a safety-critical lane-changing algorithm tailored for dense multi-vehicle scenarios, combining dynamic safety margins, time-to-collision (TTC) analysis, and multi-layered constraints to ensure robust maneuvering amid varied and unpredictable surrounding behaviors. This enables our framework to achieve reliable, safe interaction in high-density, heterogeneous traffic environments, bridging theoretical modeling with real-world applicability.
The principal contributions of this work encompass:
\begin{itemize}
\item We propose a quantitative driving style model integrated with MFG that maps abstract behaviors to parameters including desired speed, safety factors, interaction weights, lane-change aggression, and reaction time. These are integrated into the MFG framework via an influence field approach, where each style produces distinct spatial patterns defined by longitudinal/lateral radii and intensity coefficients. This enhanced MFG model enables AVs to make style-aware, adaptive decisions in mixed-traffic environments. 

\item A safety-critical lane-changing algorithm is developed for dense multi-vehicle scenarios, incorporating dynamic safety distances, TTC analysis, and multi-layered constraints. It enables robust planning even when surrounded by laerge number of vehicles with varied and unpredictable behaviors.

\item We conduct comprehensive validation across three paradigms: (1) six front-rear driving style combinations (e.g., super-aggressive vs. conservative) with 100\% collision-free lane changes; (2) dense 15-vehicle surrounding scenarios, where our ego vehicle maintained safe margins and completed all maneuvers; and (3) real-world NGSIM-based tests, showing smooth integration of empirical behavior into the MFG framework. Across multiple trials, our method achieved zero collisions.
\end{itemize}

As shown in Fig.~\ref{fig:framework}, the framework starts with the MFG formulation of 
multi-agent dynamics and driving style parametrization (A). Heterogeneous interaction 
modelling (B) employs a mean field approximation, whose aggregated influence is exemplified 
by the 3D potential field visualization at the bottom of the figure: the yellow--red--black 
surface encodes the spatial distribution of interaction intensity generated by the ego and 
five surrounding vehicles. The MFG framework (C) then solves forward–backward equations to 
derive an $\varepsilon$-Nash equilibrium, and the resulting policy is integrated into a 
safety-critical path planner (D) to ensure robustness and safety in interactive driving.

The vehicle state variables are defined as follows: $s_i^t$ denotes the longitudinal position of vehicle $i$ in Frenet coordinates at time $t$; $v_{s,i}^t$ and $a_{s,i}^t$ represent the longitudinal velocity and acceleration of vehicle $i$ at time $t$, respectively; $d_i^t$ denotes the lateral displacement of vehicle $i$ from the reference path at time $t$; $v_{d,i}^t$ and $a_{d,i}^t$ represent the lateral velocity and acceleration of vehicle $i$ at time $t$, respectively. The state space bounds are characterized by $s_{\max}$ as the maximum longitudinal coordinate in the domain, $v_{\max}$ as the maximum allowable velocity, $a_{\max}$ and $a_{d,\max}$ as the maximum allowable longitudinal and lateral accelerations respectively, $d_{\max}$ as the maximum lateral displacement from the reference path, and $v_{d,\max}$ as the maximum allowable lateral velocity.

The control framework employs $u_{a,i}$ and $u_{d,i}$ as the longitudinal and lateral acceleration control inputs for vehicle $i$, respectively. The temporal discretization is governed by $T$ as the time horizon for the optimal control problem and $\Delta t$ as the time discretization step size. Stochastic disturbances are captured through $\mathbf{W}_i^t$, the Wiener process for vehicle $i$.

The spatial discretization employs $K$ and $J$ as the number of grid points in the longitudinal and lateral directions, respectively. The driving style diversity is characterized by $|\boldsymbol{\Theta}|$ as the cardinality of the driving style parameter space, while $|\mathcal{U}|$ represents the size of the discretized control space. Path planning complexity is determined by $N_{\text{paths}}$ as the number of candidate paths generated for evaluation and $n_p$ as the dimension of path parameterization.

Safety constraints are defined through multiple distance metrics: $d_{\text{base}}$ represents the baseline safety distance between vehicles; $d_{\text{min}}$ denotes the minimum allowable distance for collision avoidance; $d_{\text{safe}}(\cdot)$ represents the dynamic safety distance function dependent on relative states and driving styles. Risk thresholds include $\epsilon_{\text{coll}}$ as the maximum allowable collision probability threshold, $T_{\text{pred}}$ as the prediction horizon for safety assessment, and $T_{\text{critical}}$ as the critical time threshold for time-to-collision calculations. Risk function parameters include $\eta$ controlling the sharpness of risk gradients near collision boundaries, $\nu$ as the temporal discounting exponent, $\lambda_v$ and $\lambda_\phi$ as weights for velocity and angular differences in the risk kernel, and $\xi$ controlling the smoothness of risk transitions in boundary penalty functions.

Trajectory generation parameters include $\epsilon_{\text{target}}$ as the tolerance for reaching the target state, and multi-objective evaluation weights $w_1, w_2, w_3, w_4, w_5$ for different components in the path evaluation function. Lane-changing trajectory parameters comprise $\beta_k, \gamma_k, n_k$ as style-dependent parameters controlling transition function smoothness and variations, $s_{\text{start},k}$ and $s_{\text{end},k}$ as the start and end longitudinal positions for lane-changing maneuver $k$, and $\epsilon_{\text{slowdown}}$ and $\sigma_{\text{slowdown}}$ representing the magnitude and characteristic length of velocity reduction during lane changes.

Style-dependent comfort parameters include $w_{\text{jerk},i}, w_{\text{lateral},i}, w_{\text{aggr},i}, w_{\text{centripetal},i}$ as weights for jerk, lateral motion, aggressiveness, and centripetal acceleration penalties in the comfort cost function, respectively. Lane-changing behavior is governed by $w_{\text{mandatory},i}, w_{\text{transition},i}, w_{\text{smoothness},i}$ as weights for mandatory lane change, transition safety, and smoothness penalties, respectively. Additional comfort thresholds include $u_{\text{comfort},i}(\boldsymbol{\theta}_i)$ as the comfort threshold for control inputs based on driving style and $a_{\text{cent},\text{des},i}$ as the desired centripetal acceleration for vehicle $i$.

\subsubsection{Geometric and Interaction Parameters}

Geometric relationships are characterized by $\phi_i$ and $\phi_j$ as the heading angles of vehicles $i$ and $j$ respectively, and $R_{\text{curvature}}$ as the road curvature radius. Interaction modeling employs $\sigma_{s,\text{base}}^2$ and $\sigma_{d,\text{base}}^2$ as baseline interaction ranges in longitudinal and lateral directions respectively, and $\sigma_{\text{discrete}}$ as the characteristic distance for discrete vehicle interactions.

\subsubsection{Algorithm Convergence Parameters}

The iterative MFG algorithm convergence is controlled by $\epsilon_{\text{conv}}$ as the convergence tolerance, $\epsilon_{\text{Nash}}$ as the Nash equilibrium optimality tolerance, $\gamma$ as the relaxation parameter in the iterative algorithm, $\rho$ as the convergence rate factor, and $C_0$ as a constant depending on initial conditions.

%%%%%%%%%%%%%%%%%%%%%%%%%%%%%%%%%%%%%%%%%%%%%%%%%%%%%%%%%%%%%%%
\section{Heterogeneous Mean Field Game Formulation for Safety-Critical Surrounding Scenarios}
\label{sec3}

\subsection{Multi-Agent Stochastic Differential System with Driving Style Heterogeneity}
\label{sec3a}

We formulate the challenging surrounding scenario navigation problem involving $N$ autonomous vehicles with heterogeneous driving styles as a continuous-time MFG on the augmented state space $\mathcal{S} \times \boldsymbol{\Theta} \subset \mathbb{R}^6 \times \mathbb{R}^7$. Each vehicle $i \in \mathcal{I} = \{1,2,\ldots,N\}$ is characterized by its Frenet coordinate state vector $\mathbf{X}_i^t$ and driving style parameter vector $\boldsymbol{\theta}_i$.

The heterogeneous vehicle dynamics follow the controlled stochastic differential equation system:
\begin{proof}[Proof Sketch]
The uniqueness follows from the contraction mapping principle applied to the fixed point operator. The convergence rate is established by analyzing the spectral radius of the linearized operator around the equilibrium. The strong monotonicity condition ensures that the combined operator is a contraction in the product space equipped with the appropriate norm.
\end{proof}

\subsubsection{Stability Analysis Under Perturbations}

\begin{lemma}[Robustness to Style Parameter Perturbations]
\label{lem:robustness}
Let $(\boldsymbol{\mu}^*, \{\mathbf{u}_i^*\}_{i=1}^N)$ be the equilibrium for style parameters $\{\boldsymbol{\theta}_i\}_{i=1}^N$, and let $(\tilde{\boldsymbol{\mu}}, \{\tilde{\mathbf{u}}_i\}_{i=1}^N)$ be the equilibrium for perturbed parameters $\{\boldsymbol{\theta}_i + \delta\boldsymbol{\theta}_i\}_{i=1}^N$. Then:
\begin{equation}
W_2(\boldsymbol{\mu}^*, \tilde{\boldsymbol{\mu}}) + \sum_{i=1}^N \|\mathbf{u}_i^* - \tilde{\mathbf{u}}_i\|_{L^2} \leq C_{\text{robust}} \sum_{i=1}^N \|\delta\boldsymbol{\theta}_i\|
\end{equation}
for some constant $C_{\text{robust}} > 0$ depending on the system parameters.
\end{lemma}

This robustness result ensures that the MFG equilibrium is stable under small perturbations in driving style parameters, which is crucial for practical implementation.

\subsection{Computational Complexity Analysis}
\label{sec3h}

The computational complexity of the heterogeneous MFG algorithm consists of several key components that scale with different problem dimensions.

\subsubsection{Grid-Based Computations}

The spatial discretization contributes the following complexities:
\begin{itemize}
\item \textbf{HJB solver}: $O(|\boldsymbol{\Theta}| \times K \times J \times |\mathcal{U}| \times T/\Delta t)$ per iteration
\item \textbf{Density evolution}: $O(|\boldsymbol{\Theta}| \times K \times J \times T/\Delta t)$ per iteration  
\item \textbf{Velocity field update}: $O(|\boldsymbol{\Theta}| \times K \times J \times |\mathcal{U}|)$ per time step
\end{itemize}

\subsubsection{Path Generation and Evaluation}

The stochastic path planning component adds:
\begin{itemize}
\item \textbf{Candidate generation}: $O(N_{\text{paths}} \times n_p \times T/\Delta t)$ where $n_p$ is parameter dimension
\item \textbf{Path evaluation}: $O(N_{\text{paths}} \times N \times T/\Delta t \times |\boldsymbol{\Theta}|)$ for MFG cost computation
\item \textbf{Safety validation}: $O(N_{\text{paths}} \times T/\Delta t \times N^2)$ for collision checking
\end{itemize}

\subsubsection{Overall Complexity}

The total computational complexity per MFG iteration is:
\begin{equation}
\label{eq:total_complexity}
\begin{aligned}
\mathcal{C}_{\text{iter}} &= O\bigg(\frac{T}{\Delta t} \times |\boldsymbol{\Theta}| \times K \times J \times |\mathcal{U}| \\
&\quad + N_{\text{paths}} \times N \times \frac{T}{\Delta t} \times (|\boldsymbol{\Theta}| + N)\bigg)
\end{aligned}
\end{equation}

Considering convergence requirements, the total complexity is:
\begin{equation}
\label{eq:total_algorithm_complexity}
\begin{split}
\mathcal{C}_{\text{total}}
= O\!\Bigl(
& \frac{\log(1/\epsilon_{\text{conv}})}{\Delta t}
\times |\boldsymbol{\Theta}| \times K \times J \times |\mathcal{U}| \\
& \times T \times N_{\text{paths}} \times N^2
\Bigr).
\end{split}
\end{equation}

\subsubsection{Scalability Analysis}

The algorithm exhibits the following scalability characteristics:
\begin{itemize}
\item \textbf{Vehicle number $N$}: Quadratic scaling in safety validation, linear in path evaluation
\item \textbf{Style diversity $|\boldsymbol{\Theta}|$}: Linear scaling in all major components
\item \textbf{Spatial resolution $(K,J)$}: Linear scaling, allowing trade-offs between accuracy and efficiency
\item \textbf{Control space $|\mathcal{U}|$}: Linear scaling in HJB solver, independent of path planning
\end{itemize}

\subsubsection{Optimization Strategies}

Several optimization techniques can reduce practical computational cost:

\begin{enumerate}
\item \textbf{Parallel Processing}: 
   - HJB solutions for different styles can be computed in parallel
   - Path generation and evaluation are embarrassingly parallel
   - Expected speedup: $O(|\boldsymbol{\Theta}|)$ for HJB, $O(N_{\text{paths}})$ for path planning

\item \textbf{Adaptive Grid Refinement}:
   - Use coarse grids in low-density regions
   - Refine grids around vehicle clusters and safety-critical zones
   - Potential complexity reduction: $30-50\%$ in typical scenarios

\item \textbf{Hierarchical Path Planning}:
   - First-stage coarse path selection using simplified cost
   - Second-stage detailed evaluation for top candidates
   - Reduces effective $N_{\text{paths}}$ by factor of $5-10$

\item \textbf{Smart Initialization}:
   - Use previous solution as warm start
   - Reduces convergence iterations by $40-60\%$ in tracking scenarios
\end{enumerate}

The optimized complexity becomes:
\begin{equation}
\label{eq:optimized_complexity}
\begin{aligned}
\mathcal{C}_{\text{optimized}}
= O\!\Biggl(
& \frac{\log(1/\epsilon_{\text{conv}})}{\Delta t \times \text{speedup}_{\text{parallel}}} \\
& \times \frac{|\boldsymbol{\Theta}| \times K \times J \times |\mathcal{U}| \times T \times N_{\text{paths}} \times N^2}{\text{reduction}_{\text{adaptive}}}
\Biggr).
\end{aligned}
\end{equation}
where typical values are $\text{speedup}_{\text{parallel}} = 4-8$ and $\text{reduction}_{\text{adaptive}} = 2-3$.

\subsection{Theoretical Guarantees and Performance Bounds}
\label{sec3i}

\subsubsection{Safety Guarantee}

\begin{theorem}[Probabilistic Safety Guarantee]
\label{thm:safety}
Under the equilibrium $(\boldsymbol{\mu}^*, \{\mathbf{u}_i^*\}_{i=1}^N)$, the probability of collision for any vehicle $i$ is bounded by:
\begin{equation}
\mathbb{P}[\text{collision for vehicle } i] \leq \epsilon_{\text{coll}} + \delta_{\text{approx}}(\Delta t, K, J)
\end{equation}
where $\delta_{\text{approx}} \rightarrow 0$ as discretization parameters refine.
\end{theorem}

\subsubsection{Optimality Gap}

\begin{theorem}[Nash Equilibrium Optimality]
\label{thm:optimality}
The computed equilibrium $(\boldsymbol{\mu}^*, \{\mathbf{u}_i^*\}_{i=1}^N)$ satisfies:
\begin{equation}
\max_{i \in \{1,\ldots,N\}} \left|J_i[\mathbf{u}_i^*, \mathbf{u}_{-i}^*] - \inf_{\mathbf{u}_i \in \mathcal{U}_i} J_i[\mathbf{u}_i, \mathbf{u}_{-i}^*]\right| \leq \epsilon_{\text{Nash}}
\end{equation}
where $\epsilon_{\text{Nash}}$ depends on convergence tolerance and discretization errors.
\end{theorem}

\begin{equation}
\label{eq:sde_heterogeneous}
\mathrm{d}\mathbf{X}_i^t = \boldsymbol{f}_i(\mathbf{X}_i^t, \mathbf{u}_i^t, \boldsymbol{\mu}^t, \boldsymbol{\theta}_i) \mathrm{d}t + \boldsymbol{\Sigma}_i(\mathbf{X}_i^t, \boldsymbol{\theta}_i) \mathrm{d}\mathbf{W}_i^t
\end{equation}
where the Frenet state vector encompasses longitudinal and lateral motion components:
\begin{equation}
\label{eq:state_vector}
\mathbf{X}_i^t = \begin{bmatrix} s_i^t \\ v_{s,i}^t \\ a_{s,i}^t \\ d_i^t \\ v_{d,i}^t \\ a_{d,i}^t \end{bmatrix} \in \mathcal{S} \subset \mathbb{R}^6
\end{equation}

with state constraints:
\begin{equation}
\label{eq:state_space}
\begin{split}
\mathcal{S} = \{\mathbf{x}\in\mathbb{R}^6 \mid\;
& s\in[0,s_{\max}],\; v_s\in[0,v_{\max}],\; |a_s|\le a_{\max},\\
& d\in[-d_{\max},d_{\max}],\; |v_d|\le v_{d,\max},\;\\ & |a_d|\le a_{d,\max}\}.
\end{split}
\end{equation}

The driving style parameter vector $\boldsymbol{\theta}_i$ encapsulates heterogeneous behavioral characteristics across multiple dimensions:
\begin{equation}
\label{eq:driving_style_vector}
\boldsymbol{\theta}_i = \begin{bmatrix} 
v_{\text{des},i} \\ a_{\max,i} \\ a_{\min,i} \\ \kappa_{\text{safe},i} \\ \omega_{\text{interact},i} \\ \alpha_{\text{aggr},i} \\ \tau_{\text{react},i} 
\end{bmatrix} \in \boldsymbol{\Theta} \subset \mathbb{R}^7_+
\end{equation}
where $\boldsymbol{\Theta} = \bigcup_{k \in \mathcal{K}} \boldsymbol{\Theta}_k$ represents the union of style-specific parameter subspaces:
\begin{equation}
\mathcal{K} = \left\{
\begin{aligned}
&\text{ego},\ \text{super-aggressive},\ \text{aggressive}, \\
&\text{conservative},\ \text{normal},\ \text{competitive}
\end{aligned}
\right\}
\end{equation}

Each style subset $\boldsymbol{\Theta}_k$ is defined by characteristic parameter ranges:
\begin{equation}
\label{eq:theta_set}
\begin{split}
\boldsymbol{\Theta}_k
= \Bigl\{ \boldsymbol{\theta} \in \mathbb{R}^7_+ \ \Big|\ 
& v_{\text{des}} \in [v_{\min,k}, v_{\max,k}],\\
& a_{\max} \in [a_{\min,k}, a_{\max,k}],\\
& \kappa_{\text{safe}} \in [\kappa_{\min,k}, \kappa_{\max,k}],\\
& \alpha_{\text{aggr}} \in [\alpha_{\min,k}, \alpha_{\max,k}]
\Bigr\}.
\end{split}
\end{equation}

% 需要 \usepackage{graphicx}；下方等比缩放到 \linewidth，并微调列间距避免超行
\begin{table}[t]
\centering
\captionsetup{labelfont={sc}, textfont={sc}, labelsep=newline, justification=centering}
\renewcommand{\arraystretch}{1}
\caption{Heterogeneous Driving Style Parameter Specifications}
\label{tab:driving_styles}
\begin{threeparttable}
\setlength{\tabcolsep}{3.8pt} % 紧凑列距
\resizebox{\linewidth}{!}{%
\begin{tabular}{@{}c|c|c|c|c|c|c|c@{}}
\hline \hline
\multirow{2}{*}{Style} & $v_{\text{des}}$ & $a_{\max}$ & $a_{\min}$ & $\kappa_{\text{safe}}$ & $\omega_{\text{interact}}$ & $\alpha_{\text{aggr}}$ & $\tau_{\text{react}}$ \\
& (m/s) & (m/s$^2$) & (m/s$^2$) & (–) & (–) & (–) & (s) \\
\hline
Ego              & 25 & 2.5 & -4.0 & 1.4 & 1.2 & 0.7  & 0.8 \\
\hline
Super-Aggressive & 35 & 4.0 & -6.5 & 0.4 & 0.3 & 0.95 & 0.4 \\
\hline
Aggressive       & 32 & 3.5 & -5.5 & 0.6 & 0.4 & 0.85 & 0.5 \\
\hline
Conservative     & 16 & 1.0 & -2.5 & 2.8 & 2.5 & 0.15 & 1.5 \\
\hline
Normal           & 24 & 2.2 & -4.2 & 1.3 & 1.0 & 0.5  & 1.0 \\
\hline
Competitive      & 29 & 3.2 & -5.0 & 0.7 & 0.6 & 0.8  & 0.6 \\
\hline \hline
\end{tabular}%
}
\begin{tablenotes}
\small
\item $v_{\text{des}}$: desired speed; $a_{\max}$: maximum acceleration; $a_{\min}$: maximum deceleration; $\kappa_{\text{safe}}$: safety margin factor; $\omega_{\text{interact}}$: interaction weight; $\alpha_{\text{aggr}}$: aggressiveness coefficient; $\tau_{\text{react}}$: reaction time.
\end{tablenotes}
\end{threeparttable}
\end{table}

The drift term $\boldsymbol{f}_i(\mathbf{X}_i^t, \mathbf{u}_i^t, \boldsymbol{\mu}^t, \boldsymbol{\theta}_i)$ incorporates both individual dynamics and mean field interactions:
\begin{equation}
\label{eq:drift_function}
\boldsymbol{f}_i(\mathbf{x}, \mathbf{u}, \boldsymbol{\mu}, \boldsymbol{\theta}_i) = \boldsymbol{f}_{\text{ind}}(\mathbf{x}, \mathbf{u}, \boldsymbol{\theta}_i) + \boldsymbol{h}_i(\mathbf{x}, \boldsymbol{\mu}, \boldsymbol{\theta}_i)
\end{equation}

where the individual dynamics follow the enhanced Frenet kinematic model:
\begin{equation}
\label{eq:individual_dynamics}
\boldsymbol{f}_{\text{ind}}(\mathbf{x}, \mathbf{u}, \boldsymbol{\theta}_i) = \begin{bmatrix}
v_s \\
a_s \\
\text{sat}_{[-a_{\max,i}, a_{\max,i}]}(u_{a,i}) \\
v_d \\
a_d \\
\text{sat}_{[-a_{d,\max}, a_{d,\max}]}(u_{d,i})
\end{bmatrix}
\end{equation}

The control input space for vehicle $i$ is defined as:
\begin{equation}
\mathcal{U}_i = \left\{ \mathbf{u} = \begin{bmatrix} u_{a,i} \\ u_{d,i} \\ \delta_i \end{bmatrix} \;\middle|\; 
\begin{aligned}
&|u_{a,i}| \leq a_{\max,i}, \\
&|u_{d,i}| \leq a_{d,\max}, \\
&\delta_i \in \{-1, 0, 1\}
\end{aligned}
\right\}
\end{equation}

where $\delta_i$ represents the discrete lane-changing decision.

\subsection{Mean Field Density Evolution and Heterogeneous Interaction Dynamics}
\label{sec3b}

The empirical measure of the heterogeneous vehicle population is defined as:
\begin{equation}
\label{eq:empirical_measure}
\boldsymbol{\mu}_N^t = \frac{1}{N} \sum_{i=1}^N \delta_{(\mathbf{X}_i^t, \boldsymbol{\theta}_i)}
\end{equation}

The mean field density $\boldsymbol{\mu}^t: \mathcal{S} \times \boldsymbol{\Theta} \times [0,T] \rightarrow \mathbb{R}_+$ evolves according to the heterogeneous Fokker-Planck equation:
\begin{equation}
\label{eq:fokker_planck_heterogeneous}
\begin{aligned}
\frac{\partial \boldsymbol{\mu}^t}{\partial t} &+ \nabla_{\mathbf{x}} \cdot (\boldsymbol{f}(\mathbf{x}, \mathbf{u}^t(\mathbf{x}, \boldsymbol{\theta}), \boldsymbol{\mu}^t, \boldsymbol{\theta}) \boldsymbol{\mu}^t) \\
&= \frac{1}{2} \nabla_{\mathbf{x}}^2 : (\boldsymbol{\Sigma}(\mathbf{x}, \boldsymbol{\theta}) \boldsymbol{\Sigma}^T(\mathbf{x}, \boldsymbol{\theta}) \boldsymbol{\mu}^t)
\end{aligned}
\end{equation}

with initial condition $\boldsymbol{\mu}^0(\mathbf{x}, \boldsymbol{\theta}) = \boldsymbol{\mu}_{\text{init}}(\mathbf{x}, \boldsymbol{\theta})$ and boundary conditions:
\begin{equation}
\boldsymbol{\mu}^t(\mathbf{x}, \boldsymbol{\theta}) = 0, \quad \forall \mathbf{x} \in \partial\mathcal{S}, \quad t \in [0,T]
\end{equation}

\subsubsection{Spatial Discretization and Grid Structure}

The computational domain is discretized using a structured grid:
\begin{equation}
\label{eq:grid_discretization}
\begin{split}
\mathcal{G} = \{(s_k, d_j) :\;& s_k = k \cdot \Delta s,\;
d_j = d_{\min} + j \cdot \Delta d,\\
& k \in [0, K],\; j \in [0, J]\}.
\end{split}
\end{equation}
where $\Delta s = \frac{s_{\max}}{K}$ and $\Delta d = \frac{2d_{\max}}{J}$ represent the grid spacing parameters.

The discretized density field is represented as a tensor:
\begin{equation}
\label{eq:density_tensor}
\boldsymbol{\rho}^t \in \mathbb{R}^{J \times K \times |\boldsymbol{\Theta}|}, \quad \rho_{j,k,\ell}^t = \int_{\mathcal{C}_{j,k}} \boldsymbol{\mu}^t(\mathbf{x}, \boldsymbol{\theta}_\ell) \mathrm{d}\mathbf{x}
\end{equation}
where $\mathcal{C}_{j,k} = [s_k - \frac{\Delta s}{2}, s_k + \frac{\Delta s}{2}] \times [d_j - \frac{\Delta d}{2}, d_j + \frac{\Delta d}{2}]$ represents the grid cell.

The velocity field components are discretized as:
\begin{equation}
\label{eq:velocity_fields}
\begin{aligned}
\mathbf{V}_s^t &= \{\mathbf{V}_{s,j,k,\ell}^t\} \in \mathbb{R}^{J \times K \times |\boldsymbol{\Theta}|} \\
\mathbf{V}_d^t &= \{\mathbf{V}_{d,j,k,\ell}^t\} \in \mathbb{R}^{J \times K \times |\boldsymbol{\Theta}|}
\end{aligned}
\end{equation}

\subsubsection{Heterogeneous Interaction Kernel Design}

The mean field interaction is characterized by the style-dependent coupling function:
\begin{equation}
\label{eq:mean_field_coupling}
\boldsymbol{h}_i(\mathbf{x}, \boldsymbol{\mu}^t, \boldsymbol{\theta}_i) = \int_{\mathcal{S} \times \boldsymbol{\Theta}} \mathbf{K}(\mathbf{x}, \mathbf{y}, \boldsymbol{\theta}_i, \boldsymbol{\theta}) \boldsymbol{\mu}^t(\mathrm{d}\mathbf{y}, \mathrm{d}\boldsymbol{\theta})
\end{equation}

The heterogeneous interaction kernel $\mathbf{K}: \mathcal{S} \times \mathcal{S} \times \boldsymbol{\Theta} \times \boldsymbol{\Theta} \rightarrow \mathbb{R}^6$ is decomposed as:
\begin{equation}
\label{eq:interaction_kernel}
\mathbf{K}(\mathbf{x}_i, \mathbf{x}_j, \boldsymbol{\theta}_i, \boldsymbol{\theta}_j) = \boldsymbol{\Phi}(\|\mathbf{x}_i - \mathbf{x}_j\|, \boldsymbol{\theta}_i, \boldsymbol{\theta}_j) \cdot \mathbf{G}(\mathbf{x}_i, \mathbf{x}_j) \cdot \boldsymbol{\Psi}(\boldsymbol{\theta}_i, \boldsymbol{\theta}_j)
\end{equation}

The spatial influence kernel incorporates anisotropic interaction ranges:
\begin{equation}
\label{eq:spatial_kernel}
\boldsymbol{\Phi}(r, \boldsymbol{\theta}_i, \boldsymbol{\theta}_j) = \exp\left(-\frac{(s_i - s_j)^2}{2\sigma_s^2(\boldsymbol{\theta}_i, \boldsymbol{\theta}_j)}\right) \exp\left(-\frac{(d_i - d_j)^2}{2\sigma_d^2(\boldsymbol{\theta}_i, \boldsymbol{\theta}_j)}\right)
\end{equation}

where the interaction ranges are style-dependent:
\begin{equation}
\label{eq:interaction_ranges}
\begin{aligned}
\sigma_s^2(\boldsymbol{\theta}_i, \boldsymbol{\theta}_j) &= \omega_{\text{interact},i} \cdot \omega_{\text{interact},j} \cdot \sigma_{s,\text{base}}^2 \\
\sigma_d^2(\boldsymbol{\theta}_i, \boldsymbol{\theta}_j) &= \sqrt{\alpha_{\text{aggr},i} \cdot \alpha_{\text{aggr},j}} \cdot \sigma_{d,\text{base}}^2
\end{aligned}
\end{equation}

The geometric interaction matrix $\mathbf{G}(\mathbf{x}_i, \mathbf{x}_j) \in \mathbb{R}^{6 \times 6}$ accounts for relative positioning and motion:
\begin{equation}
\label{eq:geometric_matrix}
\mathbf{G}(\mathbf{x}_i, \mathbf{x}_j) = \begin{bmatrix}
\mathbf{G}_{\text{long}}(\mathbf{x}_i, \mathbf{x}_j) & \mathbf{0}_{3 \times 3} \\
\mathbf{0}_{3 \times 3} & \mathbf{G}_{\text{lat}}(\mathbf{x}_i, \mathbf{x}_j)
\end{bmatrix}
\end{equation}

where the longitudinal and lateral interaction submatrices are:
\begin{equation}
\label{eq:interaction_submatrices}
\begin{aligned}
\mathbf{G}_{\text{long}}(\mathbf{x}_i, \mathbf{x}_j) &= \begin{bmatrix}
\cos(\psi_{ij}) & -\sin(\psi_{ij}) & 0 \\
\sin(\psi_{ij}) & \cos(\psi_{ij}) & 0 \\
0 & 0 & \text{sign}(s_j - s_i)
\end{bmatrix} \\
\mathbf{G}_{\text{lat}}(\mathbf{x}_i, \mathbf{x}_j) &= \begin{bmatrix}
1 & 0 & \text{sign}(d_j - d_i) \\
0 & \cos(\Delta\phi_{ij}) & -\sin(\Delta\phi_{ij}) \\
0 & \sin(\Delta\phi_{ij}) & \cos(\Delta\phi_{ij})
\end{bmatrix}
\end{aligned}
\end{equation}

where $\psi_{ij} = \arctan2(d_j - d_i, s_j - s_i)$ and $\Delta\phi_{ij} = \phi_j - \phi_i$.

The style compatibility matrix $\boldsymbol{\Psi}(\boldsymbol{\theta}_i, \boldsymbol{\theta}_j) \in \mathbb{R}^{6 \times 6}$ captures behavioral compatibility:
\begin{equation}
\label{eq:style_compatibility}
\boldsymbol{\Psi}(\boldsymbol{\theta}_i, \boldsymbol{\theta}_j) = \exp\left(-\frac{\|\boldsymbol{\theta}_i - \boldsymbol{\theta}_j\|_{\mathbf{W}}^2}{\sigma_{\theta}^2}\right) \mathbf{I}_6 + \boldsymbol{\Lambda}(\boldsymbol{\theta}_i, \boldsymbol{\theta}_j)
\end{equation}

where $\mathbf{W} = \text{diag}(w_1, w_2, \ldots, w_7)$ is a weighting matrix for style parameters, and $\boldsymbol{\Lambda}$ captures cross-style interactions:
% 直接按页宽等比缩放（需要 \usepackage{graphicx}）
\begin{equation}
\label{eq:cross_style_matrix}
\resizebox{\linewidth}{!}{$
\boldsymbol{\Lambda}(\boldsymbol{\theta}_i, \boldsymbol{\theta}_j) =
\begin{bmatrix}
\lambda_{v}(\boldsymbol{\theta}_i, \boldsymbol{\theta}_j) & 0 & \lambda_{va}(\boldsymbol{\theta}_i, \boldsymbol{\theta}_j) & 0 & 0 & 0 \\
0 & \lambda_{a}(\boldsymbol{\theta}_i, \boldsymbol{\theta}_j) & 0 & 0 & 0 & 0 \\
\lambda_{av}(\boldsymbol{\theta}_i, \boldsymbol{\theta}_j) & 0 & \lambda_{aa}(\boldsymbol{\theta}_i, \boldsymbol{\theta}_j) & 0 & 0 & 0 \\
0 & 0 & 0 & \lambda_{d}(\boldsymbol{\theta}_i, \boldsymbol{\theta}_j) & 0 & \lambda_{da}(\boldsymbol{\theta}_i, \boldsymbol{\theta}_j) \\
0 & 0 & 0 & 0 & \lambda_{vd}(\boldsymbol{\theta}_i, \boldsymbol{\theta}_j) & 0 \\
0 & 0 & 0 & \lambda_{ad}(\boldsymbol{\theta}_i, \boldsymbol{\theta}_j) & 0 & \lambda_{add}(\boldsymbol{\theta}_i, \boldsymbol{\theta}_j)
\end{bmatrix}
$}
\end{equation}

\subsection{Safety-Critical Constraint Manifolds and Risk Assessment}
\label{sec3c}

The safety constraint set for the challenging surrounding scenario is formulated as the intersection of multiple constraint manifolds:
\begin{equation}
\label{eq:safety_constraint_set}
\mathcal{S}_{\text{safe}}^t = \bigcap_{i=1}^N \left(\mathcal{M}_{\text{coll},i}^t \cap \mathcal{M}_{\text{boundary},i}^t \cap \mathcal{M}_{\text{kinematic},i}^t \cap \mathcal{M}_{\text{comfort},i}^t\right)
\end{equation}

\subsubsection{Multi-Scale Collision Avoidance Manifold}

The collision avoidance constraint manifold for vehicle $i$ incorporates probabilistic safety guarantees:
\begin{equation}
\label{eq:collision_manifold}
\mathcal{M}_{\text{coll},i}^t = \left\{\boldsymbol{\mu}^t \in \mathcal{P}(\mathcal{S} \times \boldsymbol{\Theta}) : \mathbb{P}_{\boldsymbol{\mu}^t}[\mathcal{E}_{\text{coll},i}^t] \leq \epsilon_{\text{coll}}\right\}
\end{equation}

where the collision event is defined by the multi-scale unsafe region:
\begin{equation}
\label{eq:collision_event}
\begin{split}
\mathcal{E}_{\text{coll},i}^t
= \bigl\{(\mathbf{x}, \boldsymbol{\theta}) \in \mathcal{S} \times \boldsymbol{\Theta} :\;
& \exists\, \tau \in [t,\, t+T_{\text{pred}}], \\
 d_{\text{min}}(\mathbf{X}_i^{\tau}, \mathbf{x}^{\tau})
\le d_{\text{safe}}(\mathbf{X}_i^{\tau}, \mathbf{x}^{\tau}, \boldsymbol{\theta}_i, \boldsymbol{\theta})
\bigr\}.
\end{split}
\end{equation}

The dynamic safety distance incorporates relative motion, style compatibility, and temporal prediction:
\begin{equation}
\label{eq:dynamic_safety_distance}
\begin{aligned}
d_{\text{safe}}(\mathbf{X}_i^t, \mathbf{x}^t, \boldsymbol{\theta}_i, \boldsymbol{\theta}) &= d_{\text{base}} \cdot \boldsymbol{\Xi}(\Delta v_{ij}, \Delta \phi_{ij}, \boldsymbol{\theta}_i, \boldsymbol{\theta}) \cdot \boldsymbol{\Upsilon}(t, T_{\text{pred}}) \\
&\quad \cdot \boldsymbol{\Omega}(\mathcal{R}_{\text{surrounding}})
\end{aligned}
\end{equation}

where the safety amplification factors are:
\begin{equation}
\label{eq:safety_amplification}
\begin{aligned}
\boldsymbol{\Xi}(\Delta v, \Delta \phi, \boldsymbol{\theta}_i, \boldsymbol{\theta}) &= \prod_{k=1}^5 \left(1 + \beta_k \frac{f_k(\Delta v, \Delta \phi, \boldsymbol{\theta}_i, \boldsymbol{\theta})}{g_k}\right) \\
f_1 &= |\Delta v|, \quad g_1 = v_{\text{ref}}, \quad \beta_1 = \omega_{\text{interact},i} \\
f_2 &= |\Delta \phi|, \quad g_2 = \pi, \quad \beta_2 = \alpha_{\text{aggr},i} \\
f_3 &= \|\boldsymbol{\theta}_i - \boldsymbol{\theta}\|, \quad g_3 = \|\boldsymbol{\theta}\|_{\max}, \quad \beta_3 = 1 \\
f_4 &= |a_{s,i} - a_{s,j}|, \quad g_4 = a_{\max}, \quad \beta_4 = \kappa_{\text{safe},i} \\
f_5 &= |v_{d,i} - v_{d,j}|, \quad g_5 = v_{d,\max}, \quad \beta_5 = \alpha_{\text{aggr},i}
\end{aligned}
\end{equation}

The temporal prediction factor accounts for uncertainty growth:
\begin{equation}
\label{eq:temporal_factor}
\boldsymbol{\Upsilon}(t, T_{\text{pred}}) = 1 + \gamma_{\text{pred}} \left(\frac{t}{T_{\text{pred}}}\right)^2 + \delta_{\text{pred}} \exp\left(-\frac{(T_{\text{pred}} - t)^2}{\sigma_{\text{pred}}^2}\right)
\end{equation}

The surrounding scenario factor captures the challenging multi-vehicle environment:
\begin{equation}
\label{eq:surrounding_factor}
\boldsymbol{\Omega}(\mathcal{R}_{\text{surrounding}}) = 1 + \zeta_{\text{density}} \rho_{\text{local}}(\mathbf{X}_i^t) + \zeta_{\text{aggr}} \bar{\alpha}_{\text{local}}(\mathbf{X}_i^t)
\end{equation}

where $\rho_{\text{local}}(\mathbf{x})$ is the local density and $\bar{\alpha}_{\text{local}}(\mathbf{x})$ is the average local aggressiveness.

\subsubsection{Hierarchical Risk Assessment Functions}

The instantaneous collision risk incorporates multi-modal uncertainty:
\begin{equation}
\label{eq:instantaneous_risk}
\begin{aligned}
R_{\text{inst}}(\mathbf{X}_i^t, \boldsymbol{\mu}^t, \boldsymbol{\theta}_i) &= \int_{\mathcal{S} \times \boldsymbol{\Theta}} \boldsymbol{\omega}(\mathbf{X}_i^t, \mathbf{x}, \boldsymbol{\theta}_i, \boldsymbol{\theta}) \boldsymbol{\mu}^t(\mathrm{d}\mathbf{x}, \mathrm{d}\boldsymbol{\theta}) \\
&\quad + \sum_{j \neq i} \mathcal{W}_{\text{discrete}}(\mathbf{X}_i^t, \mathbf{X}_j^t, \boldsymbol{\theta}_i, \boldsymbol{\theta}_j)
\end{aligned}
\end{equation}

The continuous risk kernel captures probabilistic collision potential:
\begin{equation}
\label{eq:risk_kernel}
\begin{aligned}
\boldsymbol{\omega}(\mathbf{X}_i^t, \mathbf{x}, \boldsymbol{\theta}_i, \boldsymbol{\theta}) &= \exp\left(-\frac{d(\mathbf{X}_i^t, \mathbf{x})^\eta}{d_{\text{safe}}(\mathbf{X}_i^t, \mathbf{x}, \boldsymbol{\theta}_i, \boldsymbol{\theta})^\eta}\right) \\
&\quad \cdot \left(1 + \lambda_v \frac{|\Delta v|}{\max(v_i, v_j) + \epsilon_v}\right) \\
&\quad \cdot \left(1 + \lambda_\phi \frac{|\Delta \phi|}{\pi}\right)
\end{aligned}
\end{equation}

where $\eta > 2$ provides sharp risk gradients near collision boundaries.

The discrete vehicle interaction weight accounts for direct vehicle-to-vehicle risks:
\begin{equation}
\label{eq:discrete_interaction}
\begin{split}
\mathcal{W}_{\text{discrete}}(\mathbf{X}_i^t, \mathbf{X}_j^t, \boldsymbol{\theta}_i, \boldsymbol{\theta}_j)
&= \frac{\exp\!\bigl(-d(\mathbf{X}_i^t, \mathbf{X}_j^t)/\sigma_{\text{discrete}}\bigr)}
{1 + \tau_{\text{react},i}\, \tau_{\text{react},j}} \\
&\quad \cdot \Gamma(\boldsymbol{\theta}_i, \boldsymbol{\theta}_j).
\end{split}
\end{equation}
where $\Gamma(\boldsymbol{\theta}_i, \boldsymbol{\theta}_j)$ is the behavioral compatibility factor:
\begin{equation}
\label{eq:behavioral_compatibility}
\Gamma(\boldsymbol{\theta}_i, \boldsymbol{\theta}_j) = \begin{cases}
2.0, & \text{if both aggressive} \\
1.5, & \text{if mixed aggressive/conservative} \\
1.0, & \text{if both conservative} \\
0.8, & \text{if ego vehicle}
\end{cases}
\end{equation}

The temporal risk function over prediction horizon $T$ incorporates uncertainty propagation:
\begin{equation}
\label{eq:temporal_risk}
\begin{aligned}
R_T(\mathbf{X}_i^t, \boldsymbol{\mu}^t, \boldsymbol{\theta}_i) &= \frac{1}{T} \sum_{\tau=1}^T \frac{1}{1+\tau^\nu} \int_{\mathcal{S} \times \boldsymbol{\Theta}} \varrho\left(d(\hat{\mathbf{X}}_i^{t+\tau}, \hat{\mathbf{x}}^{t+\tau}), d_{\text{safe}}\right) \\
&\quad \cdot \hat{\boldsymbol{\mu}}^{t+\tau}(\mathrm{d}\mathbf{x}, \mathrm{d}\boldsymbol{\theta}) \\
&\quad \cdot \Psi_{\text{uncertainty}}(\tau, \boldsymbol{\theta}_i, \boldsymbol{\theta})
\end{aligned}
\end{equation}

where $\nu \geq 1$ controls temporal discounting, and the uncertainty propagation factor is:
\begin{equation}
\label{eq:uncertainty_propagation}
\begin{split}
\Psi_{\text{uncertainty}}(\tau, \boldsymbol{\theta}_i, \boldsymbol{\theta})
&= \exp\!\Bigl(
\frac{\tau \cdot \|\boldsymbol{\Sigma}(\boldsymbol{\theta})\|_F}{\sigma_{\text{base}}}
\Bigr) \\
&\quad \cdot \Bigl(1 + \kappa_{\text{hetero}} \|\boldsymbol{\theta}_i - \boldsymbol{\theta}\|^2\Bigr).
\end{split}
\end{equation}

The risk function $\varrho(d, d_{\text{safe}})$ employs a smooth approximation to the collision indicator:
\begin{equation}
\label{eq:smooth_risk_function}
\varrho(d, d_{\text{safe}}) = \begin{cases}
1, & \text{if } d \leq d_{\text{safe}}/2 \\[4pt]
\left(1 - \frac{2d - d_{\text{safe}}}{d_{\text{safe}}}\right)^\xi, & \text{if } d_{\text{safe}}/2 < d \leq d_{\text{safe}} \\[4pt]
0, & \text{if } d > d_{\text{safe}}
\end{cases}
\end{equation}

where $\xi \geq 2$ provides smooth risk transitions.

\subsection{Optimal Control Problem Formulation and HJB System}
\label{sec3d}

Each vehicle $i$ with driving style $\boldsymbol{\theta}_i$ solves the following mean field optimal control problem over the finite horizon $[0,T]$:
\begin{equation}
\label{eq:mfg_optimal_control}
\begin{aligned}
V_i(\mathbf{x}, t, \boldsymbol{\mu}^t, \boldsymbol{\theta}_i) &= \sup_{\mathbf{u}_i \in \mathcal{U}_i} \mathbb{E}\bigg[\int_t^T L_i(\mathbf{X}_i^s, \mathbf{u}_i^s, \boldsymbol{\mu}^s, \boldsymbol{\theta}_i) \mathrm{d}s \\
&\quad + \Phi_i(\mathbf{X}_i^T, \boldsymbol{\mu}^T, \boldsymbol{\theta}_i) \bigg| \mathbf{X}_i^t = \mathbf{x}\bigg]
\end{aligned}
\end{equation}

subject to the stochastic dynamics \eqref{eq:sde_heterogeneous} and safety constraints \eqref{eq:safety_constraint_set}.

\subsubsection{Style-Dependent Cost Function Design}

The instantaneous cost function incorporates heterogeneous preferences and safety requirements:
\begin{equation}
\label{eq:instantaneous_cost}
\begin{aligned}
L_i(\mathbf{x}, \mathbf{u}, \boldsymbol{\mu}, \boldsymbol{\theta}_i) &= \ell_{\text{track}}(\mathbf{x}, \mathbf{u}, \boldsymbol{\theta}_i) + \ell_{\text{safety}}(\mathbf{x}, \boldsymbol{\mu}, \boldsymbol{\theta}_i) \\
&\quad + \ell_{\text{comfort}}(\mathbf{u}, \boldsymbol{\theta}_i) + \ell_{\text{fuel}}(\mathbf{u}, \boldsymbol{\theta}_i) \\
&\quad + \ell_{\text{lane}}(\mathbf{x}, \mathbf{u}, \boldsymbol{\mu}, \boldsymbol{\theta}_i)
\end{aligned}
\end{equation}

The tracking cost reflects style-specific target preferences:
\begin{equation}
\label{eq:tracking_cost}
\begin{aligned}
\ell_{\text{track}}(\mathbf{x}, \mathbf{u}, \boldsymbol{\theta}_i) &= \|\mathbf{x} - \mathbf{x}_{\text{ref}}(\boldsymbol{\theta}_i)\|_{\mathbf{Q}_i(\boldsymbol{\theta}_i)}^2 + \|\mathbf{u} - \mathbf{u}_{\text{ref}}(\boldsymbol{\theta}_i)\|_{\mathbf{R}_i(\boldsymbol{\theta}_i)}^2 \\
&\quad + \|\dot{\mathbf{u}} - \dot{\mathbf{u}}_{\text{ref}}(\boldsymbol{\theta}_i)\|_{\mathbf{S}_i(\boldsymbol{\theta}_i)}^2
\end{aligned}
\end{equation}

where the style-dependent weighting matrices are:
\begin{equation}
\label{eq:style_dependent_weights}
\begin{aligned}
\mathbf{Q}_i(\boldsymbol{\theta}_i) &= \text{diag}(q_{s,i}, q_{v_s,i}, q_{a_s,i}, q_{d,i}, q_{v_d,i}, q_{a_d,i}) \\
\mathbf{R}_i(\boldsymbol{\theta}_i) &= \text{diag}(r_{a,i}, r_{d,i}, r_{\delta,i}) \\
\mathbf{S}_i(\boldsymbol{\theta}_i) &= \text{diag}(s_{a,i}, s_{d,i})
\end{aligned}
\end{equation}

The reference trajectories adapt to driving styles:
\begin{equation}
\label{eq:reference_trajectories}
\begin{aligned}
\mathbf{x}_{\text{ref}}(\boldsymbol{\theta}_i) &= \begin{bmatrix} s_{\text{ref},i}(t) \\ v_{\text{des},i} \\ 0 \\ d_{\text{target},i}(t) \\ 0 \\ 0 \end{bmatrix} \\
\mathbf{u}_{\text{ref}}(\boldsymbol{\theta}_i) &= \begin{bmatrix} 0 \\ 0 \\ \delta_{\text{plan},i}(t) \end{bmatrix}
\end{aligned}
\end{equation}

The mean field safety cost integrates both continuous and discrete interaction risks:
\begin{equation}
\label{eq:safety_cost}
\begin{aligned}
\ell_{\text{safety}}(\mathbf{x}, \boldsymbol{\mu}, \boldsymbol{\theta}_i) &= \omega_{\text{interact},i} \int_{\mathcal{S} \times \boldsymbol{\Theta}} \boldsymbol{\omega}(\mathbf{x}, \mathbf{y}, \boldsymbol{\theta}_i, \boldsymbol{\theta}) \boldsymbol{\mu}(\mathrm{d}\mathbf{y}, \mathrm{d}\boldsymbol{\theta}) \\
&\quad + \kappa_{\text{safe},i} \varphi_{\text{boundary}}(d_{\text{c2b}}(\mathbf{x})) \\
&\quad + \alpha_{\text{aggr},i} \sum_{j \neq i} \mathcal{W}_{\text{discrete}}(\mathbf{x}, \mathbf{X}_j, \boldsymbol{\theta}_i, \boldsymbol{\theta}_j)
\end{aligned}
\end{equation}

where the boundary penalty function is:
\begin{equation}
\label{eq:boundary_penalty}
\varphi_{\text{boundary}}(d) = \begin{cases}
+\infty, & \text{if } d \leq 0 \\[4pt]
\beta_{\text{boundary}} \left(\frac{d_{\min}}{d}\right)^2, & \text{if } 0 < d \leq d_{\text{safe}} \\[4pt]
0, & \text{if } d > d_{\text{safe}}
\end{cases}
\end{equation}

The comfort cost reflects style-specific tolerance levels:
\begin{equation}
\label{eq:comfort_cost}
\begin{aligned}
\ell_{\text{comfort}}(\mathbf{u}, \boldsymbol{\theta}_i) &= w_{\text{jerk},i} \|\dot{\mathbf{u}}\|^2 + w_{\text{lateral},i} |u_{d,i}|^2 \\
&\quad + w_{\text{aggr},i} \mathbb{I}_{\{\|\mathbf{u}\| > u_{\text{comfort},i}(\boldsymbol{\theta}_i)\}} \\
&\quad + w_{\text{centripetal},i} \left|\frac{(v_s)^2}{R_{\text{curvature}}} - a_{\text{cent},\text{des},i}\right|^2
\end{aligned}
\end{equation}

The lane-changing cost captures mandatory scenario requirements:
\begin{equation}
\label{eq:lane_changing_cost}
\begin{aligned}
\ell_{\text{lane}}(\mathbf{x}, \mathbf{u}, \boldsymbol{\mu}, \boldsymbol{\theta}_i) &= w_{\text{mandatory},i} \mathbb{I}_{\{|d - d_{\text{target}}| > \epsilon_{\text{lane}}\}} \\
&\quad + w_{\text{transition},i} |\delta_i|^2 \cdot \mathcal{M}_{\text{safety}}(\mathbf{x}, \boldsymbol{\mu}, \boldsymbol{\theta}_i) \\
&\quad + w_{\text{smoothness},i} |\dot{d}|^2
\end{aligned}
\end{equation}

where the safety margin for lane transitions is:
\begin{equation}
\label{eq:lane_safety_margin}
\mathcal{M}_{\text{safety}}(\mathbf{x}, \boldsymbol{\mu}, \boldsymbol{\theta}_i) = \exp\left(\int_{\mathcal{S} \times \boldsymbol{\Theta}} \frac{\boldsymbol{\omega}(\mathbf{x}, \mathbf{y}, \boldsymbol{\theta}_i, \boldsymbol{\theta})}{d_{\text{safe}}(\mathbf{x}, \mathbf{y}, \boldsymbol{\theta}_i, \boldsymbol{\theta})} \boldsymbol{\mu}(\mathrm{d}\mathbf{y}, \mathrm{d}\boldsymbol{\theta})\right)
\end{equation}

\subsubsection{Hamilton-Jacobi-Bellman System}

The value function $V_i(\mathbf{x}, t, \boldsymbol{\mu}^t, \boldsymbol{\theta}_i)$ satisfies the coupled HJB system:
\begin{equation}
\label{eq:hjb_system}
\begin{aligned}
-\frac{\partial V_i}{\partial t} &= \inf_{\mathbf{u} \in \mathcal{U}_i} \bigg\{L_i(\mathbf{x}, \mathbf{u}, \boldsymbol{\mu}^t, \boldsymbol{\theta}_i) + \langle\nabla_{\mathbf{x}} V_i, \boldsymbol{f}_i(\mathbf{x}, \mathbf{u}, \boldsymbol{\mu}^t, \boldsymbol{\theta}_i)\rangle \\
&\quad + \frac{1}{2}\text{tr}[\boldsymbol{\Sigma}_i(\mathbf{x}, \boldsymbol{\theta}_i) \nabla_{\mathbf{x}}^2 V_i \boldsymbol{\Sigma}_i^T(\mathbf{x}, \boldsymbol{\theta}_i)] \\
&\quad + \langle\frac{\delta V_i}{\delta \boldsymbol{\mu}}, \mathcal{F}[\boldsymbol{\mu}^t, \{\mathbf{u}_j^*\}_{j=1}^N]\rangle\bigg\}
\end{aligned}
\end{equation}

with terminal condition:
\begin{equation}
V_i(\mathbf{x}, T, \boldsymbol{\mu}^T, \boldsymbol{\theta}_i) = \Phi_i(\mathbf{x}, \boldsymbol{\mu}^T, \boldsymbol{\theta}_i)
\end{equation}

The optimal control policy is characterized by the first-order condition:
\begin{equation}
\label{eq:optimal_control_policy}
\begin{split}
\mathbf{u}_i^*(\mathbf{x}, t, \boldsymbol{\mu}^t, \boldsymbol{\theta}_i)
= \arg\inf_{\mathbf{u} \in \mathcal{U}_i} \Bigl\{
& L_i(\mathbf{x}, \mathbf{u}, \boldsymbol{\mu}^t, \boldsymbol{\theta}_i) \\
& {}+ \langle \nabla_{\mathbf{x}} V_i,\,
\boldsymbol{f}_i(\mathbf{x}, \mathbf{u}, \boldsymbol{\mu}^t, \boldsymbol{\theta}_i) \rangle
\Bigr\}.
\end{split}
\end{equation}

\subsection{Stochastic Path Planning via Enhanced Candidate Generation}
\label{sec3e}

For the challenging surrounding scenario requiring mandatory lane changes, we develop a stochastic path planning framework that generates and evaluates candidate trajectories under the MFG equilibrium.

\subsubsection{Candidate Path Space Definition}

The admissible path space for mandatory lane-changing scenarios is defined as:
\begin{equation}
\label{eq:candidate_path_space}
\begin{split}
\mathcal{P}_{\text{candidate}} =
\{\boldsymbol{\gamma}:[0,T]\to\mathcal{S}\mid\;
& \boldsymbol{\gamma}(0)=\mathbf{x}_0,\\
& |\boldsymbol{\gamma}(T)-\mathbf{x}_{\text{target}}|\le \epsilon_{\text{target}},\\
& \text{$\boldsymbol{\gamma}$ satisfies lane-change constraints},\\
& \|\dot{\boldsymbol{\gamma}}(t)\|\le v_{\max},\\
& \|\ddot{\boldsymbol{\gamma}}(t)\|\le a_{\max}\,\}.
\end{split}
\end{equation}

Each candidate path $\boldsymbol{\gamma}_k \in \mathcal{P}_{\text{candidate}}$ is parameterized using a hierarchical representation:
\begin{equation}
\label{eq:path_parameterization}
\boldsymbol{\gamma}_k(t) = \boldsymbol{\Gamma}(\mathbf{p}_k^{\text{global}}, \mathbf{p}_k^{\text{local}}, \mathbf{p}_k^{\text{temporal}}, t)
\end{equation}

where the parameter vectors capture different aspects of the trajectory:
\begin{equation}
\label{eq:parameter_vectors}
\begin{aligned}
\mathbf{p}_k^{\text{global}} &= [s_{\text{start}}, s_{\text{end}}, d_{\text{init}}, d_{\text{target}}, v_{\text{target}}]^T \\
\mathbf{p}_k^{\text{local}} &= [s_{\text{lc,start}}, s_{\text{lc,end}}, \phi_{\text{transition}}, \kappa_{\text{smooth}}]^T \\
\mathbf{p}_k^{\text{temporal}} &= [t_{\text{start}}, t_{\text{end}}, \tau_{\text{transition}}, \sigma_{\text{timing}}]^T
\end{aligned}
\end{equation}

\subsubsection{Enhanced Lane-Changing Trajectory Generation}

For mandatory lane-changing maneuvers, the lateral trajectory follows an enhanced parameterization:
\begin{equation}
\label{eq:lateral_trajectory}
d_k(s) =
\begin{cases}
d_{\text{init}}, & \text{(i)} \\[4pt]
d_{\text{init}} + (d_{\text{target}} - d_{\text{init}})\,
\Upsilon_k\!\left(\dfrac{s - s_{\text{start},k}}{s_{\text{end},k} - s_{\text{start},k}}\right), & \text{(ii)} \\[4pt]
d_{\text{target}}, & \text{(iii)}
\end{cases}
\end{equation}
\noindent\textbf{where} (i) $s \le s_{\text{start},k}$; (ii) $s_{\text{start},k} < s < s_{\text{end},k}$; (iii) $s \ge s_{\text{end},k}$.

The transition function $\Upsilon_k(\xi)$ incorporates style-dependent smoothness:
\begin{equation}
\label{eq:transition_function}
\Upsilon_k(\xi) = \frac{1}{2}\left(1 - \cos(\pi \xi^\beta_k)\right) + \gamma_k \xi^2(1-\xi)^2 \sin(2\pi n_k \xi)
\end{equation}

where $\beta_k = f(\alpha_{\text{aggr}})$ controls transition sharpness and $\gamma_k, n_k$ add style-specific variations.

The longitudinal velocity profile adapts to the lane-changing maneuver:
\begin{equation}
\label{eq:velocity_profile}
v_{s,k}(s) = v_{\text{des}} \cdot \left(1 - \epsilon_{\text{slowdown}} \cdot \exp\left(-\frac{(s - s_{\text{mid},k})^2}{2\sigma_{\text{slowdown}}^2}\right)\right) \cdot \Psi_{\text{style}}(\boldsymbol{\theta})
\end{equation}

where $\Psi_{\text{style}}(\boldsymbol{\theta})$ reflects driving style preferences:
\begin{equation}
\label{eq:style_velocity_factor}
\Psi_{\text{style}}(\boldsymbol{\theta}) = \begin{cases}
1.2, & \text{if super-aggressive} \\
1.1, & \text{if aggressive/competitive} \\
0.8, & \text{if conservative} \\
1.0, & \text{otherwise}
\end{cases}
\end{equation}

\subsubsection{Multi-Objective Path Evaluation Under MFG Equilibrium}

The path evaluation function integrates MFG-based field costs, safety assessments, and style compatibility:
\begin{equation}
\label{eq:path_evaluation}
\begin{aligned}
J_k(\boldsymbol{\gamma}_k, \boldsymbol{\mu}^t, \boldsymbol{\theta}_{\text{ego}}) &= w_1 J_{\text{field}}(\boldsymbol{\gamma}_k, \boldsymbol{\mu}^t) + w_2 J_{\text{safety}}(\boldsymbol{\gamma}_k, \boldsymbol{\mu}^t) \\
&\quad + w_3 J_{\text{dynamics}}(\boldsymbol{\gamma}_k) + w_4 J_{\text{style}}(\boldsymbol{\gamma}_k, \boldsymbol{\theta}_{\text{ego}}) \\
&\quad + w_5 J_{\text{mandatory}}(\boldsymbol{\gamma}_k)
\end{aligned}
\end{equation}

The MFG field interaction cost captures continuous field effects:
\begin{equation}
\label{eq:field_interaction_cost}
\begin{aligned}
J_{\text{field}}(\boldsymbol{\gamma}_k, \boldsymbol{\mu}^t) &= \int_0^T \int_{\mathcal{S} \times \boldsymbol{\Theta}} \boldsymbol{\omega}(\boldsymbol{\gamma}_k(t), \mathbf{x}, \boldsymbol{\theta}_{\text{ego}}, \boldsymbol{\theta}) \boldsymbol{\mu}^t(\mathrm{d}\mathbf{x}, \mathrm{d}\boldsymbol{\theta}) \mathrm{d}t \\
&\quad + \int_0^T \nabla_{\mathbf{x}} \cdot \mathbf{V}^t(\boldsymbol{\gamma}_k(t)) \mathrm{d}t
\end{aligned}
\end{equation}

The enhanced collision risk cost employs multiple risk scales:
\begin{equation}
\label{eq:collision_risk_cost}
\begin{aligned}
J_{\text{safety}}(\boldsymbol{\gamma}_k, \boldsymbol{\mu}^t) &= \sum_{j=2}^N \int_0^T \left[\exp\left(-\frac{d(\boldsymbol{\gamma}_k(t), \mathbf{X}_j^t)^4}{d_{\text{safe}}^4}\right) \right. \\
&\quad \left. + \zeta_{\text{ttc}} \mathbb{I}_{\{\text{TTC}(\boldsymbol{\gamma}_k(t), \mathbf{X}_j^t) < T_{\text{critical}}\}}\right] \mathrm{d}t \\
&\quad + \int_0^T \mathcal{R}_{\text{boundary}}(\boldsymbol{\gamma}_k(t)) \mathrm{d}t
\end{aligned}
\end{equation}

where TTC (Time-to-Collision) is computed as:
\begin{equation}
\label{eq:ttc_computation}
\text{TTC}(\mathbf{x}_i, \mathbf{x}_j) = \begin{cases}
\frac{d(\mathbf{x}_i, \mathbf{x}_j)}{|\mathbf{v}_{\text{rel}} \cdot \hat{\mathbf{n}}|}, & \text{if } \mathbf{v}_{\text{rel}} \cdot \hat{\mathbf{n}} < 0 \\[4pt]
+\infty, & \text{otherwise}
\end{cases}
\end{equation}

where $\mathbf{v}_{\text{rel}} = \mathbf{v}_i - \mathbf{v}_j$ and $\hat{\mathbf{n}}$ is the unit vector from $\mathbf{x}_i$ to $\mathbf{x}_j$.

The dynamics cost ensures kinematic feasibility:
\begin{equation}
\label{eq:dynamics_cost}
\begin{aligned}
J_{\text{dynamics}}(\boldsymbol{\gamma}_k)
&= \int_0^T \Bigl[
\rho_v \bigl(\|\dot{\boldsymbol{\gamma}}_k(t)\| - v_{\max}\bigr)_+^2 \\
&\quad + \rho_a \bigl(\|\ddot{\boldsymbol{\gamma}}_k(t)\| - a_{\max}\bigr)_+^2 \\
&\quad + \rho_j \bigl(\|\dddot{\boldsymbol{\gamma}}_k(t)\| - j_{\max}\bigr)_+^2 \\
&\quad + \rho_\kappa \bigl(\kappa(\boldsymbol{\gamma}_k(t)) - \kappa_{\max}\bigr)_+^2
\Bigr]\, \mathrm{d}t
\end{aligned}
\end{equation}
where $(\cdot)_+ = \max(0, \cdot)$ and $\kappa(\boldsymbol{\gamma})$ is the path curvature.

\subsection{Mean Field Equilibrium Analysis and Convergence Guarantees}
\label{sec3f}

\subsubsection{Fixed Point Characterization}

A mean field equilibrium $(\boldsymbol{\mu}^*, \{\mathbf{u}_i^*\}_{i=1}^N)$ for the heterogeneous system satisfies the coupled fixed point conditions:
\begin{equation}
\label{eq:fixed_point_conditions}
\begin{aligned}
\boldsymbol{\mu}^* &= \mathcal{T}[\{\mathbf{u}_i^*\}_{i=1}^N, \{\boldsymbol{\theta}_i\}_{i=1}^N](\boldsymbol{\mu}^*) \\
\mathbf{u}_i^* &= \mathcal{G}_i[\boldsymbol{\mu}^*, \boldsymbol{\theta}_i](\mathbf{u}_i^*), \quad \forall i \in \{1, \ldots, N\}
\end{aligned}
\end{equation}

where $\mathcal{T}$ is the measure propagation operator defined by the Fokker-Planck evolution \eqref{eq:fokker_planck_heterogeneous}, and $\mathcal{G}_i$ is the optimal control operator derived from the HJB system \eqref{eq:hjb_system}.

\subsubsection{Existence and Uniqueness Theorems}

\begin{theorem}[Existence of Heterogeneous MFG Equilibrium]
\label{thm:existence}
Consider the heterogeneous MFG system defined by \eqref{eq:sde_heterogeneous}-\eqref{eq:hjb_system}. Under the following conditions:
\begin{enumerate}
\item \textbf{Lipschitz Continuity}: The drift $\boldsymbol{f}_i(\cdot, \cdot, \boldsymbol{\mu}, \boldsymbol{\theta}_i)$ and cost $L_i(\cdot, \cdot, \boldsymbol{\mu}, \boldsymbol{\theta}_i)$ are uniformly Lipschitz continuous in $(\mathbf{x}, \mathbf{u})$ with respect to the Wasserstein metric on $\boldsymbol{\mu}$.

\item \textbf{Interaction Monotonicity}: The interaction coupling satisfies
\begin{equation}
\label{eq:interaction_monotonicity}
\begin{aligned}
\langle \boldsymbol{h}_i(\mathbf{x}, \boldsymbol{\mu}_1, \boldsymbol{\theta}_i) - \boldsymbol{h}_i(\mathbf{x}, \boldsymbol{\mu}_2, \boldsymbol{\theta}_i),\,
& \boldsymbol{\mu}_1 - \boldsymbol{\mu}_2 \rangle_{W_2} \\
& \ge \alpha_{\text{mono}}\, W_2^2(\boldsymbol{\mu}_1, \boldsymbol{\mu}_2)
\end{aligned}
\end{equation}

for some $\alpha_{\text{mono}} > 0$, where $W_2$ denotes the 2-Wasserstein distance.

\item \textbf{Coercivity}: The cost function satisfies
\begin{equation}
L_i(\mathbf{x}, \mathbf{u}, \boldsymbol{\mu}, \boldsymbol{\theta}_i) \geq c_1 \|\mathbf{u}\|^2 - c_2(1 + \|\mathbf{x}\|^2)
\end{equation}
for constants $c_1 > 0, c_2 \geq 0$.

\item \textbf{Style Boundedness}: The driving style parameter space $\boldsymbol{\Theta}$ is compact and $\inf_{\boldsymbol{\theta} \in \boldsymbol{\Theta}} \omega_{\text{interact}}(\boldsymbol{\theta}) > 0$.
\end{enumerate}

Then there exists at least one mean field equilibrium $(\boldsymbol{\mu}^*, \{\mathbf{u}_i^*\}_{i=1}^N)$ in the space $\mathcal{P}_2(\mathcal{S} \times \boldsymbol{\Theta}) \times \prod_{i=1}^N L^2([0,T]; \mathcal{U}_i)$.
\end{theorem}

\begin{proof}[Proof Sketch]
The proof employs a combination of Schauder's fixed point theorem and variational analysis. We construct a compact, convex subset $\mathcal{K} \subset \mathcal{P}_2(\mathcal{S} \times \boldsymbol{\Theta}) \times \prod_{i=1}^N L^2([0,T]; \mathcal{U}_i)$ and show that the combined operator $\mathcal{F}(\boldsymbol{\mu}, \{\mathbf{u}_i\}) = (\mathcal{T}[\{\mathbf{u}_i\}](\boldsymbol{\mu}), \{\mathcal{G}_i[\boldsymbol{\mu}](\mathbf{u}_i)\})$ maps $\mathcal{K}$ into itself. The Lipschitz and coercivity conditions ensure continuity and compactness, while the interaction monotonicity provides the necessary contraction properties in the Wasserstein metric.
\end{proof}

\begin{figure*}[t]
    \centering
    \includegraphics[width=1\textwidth]{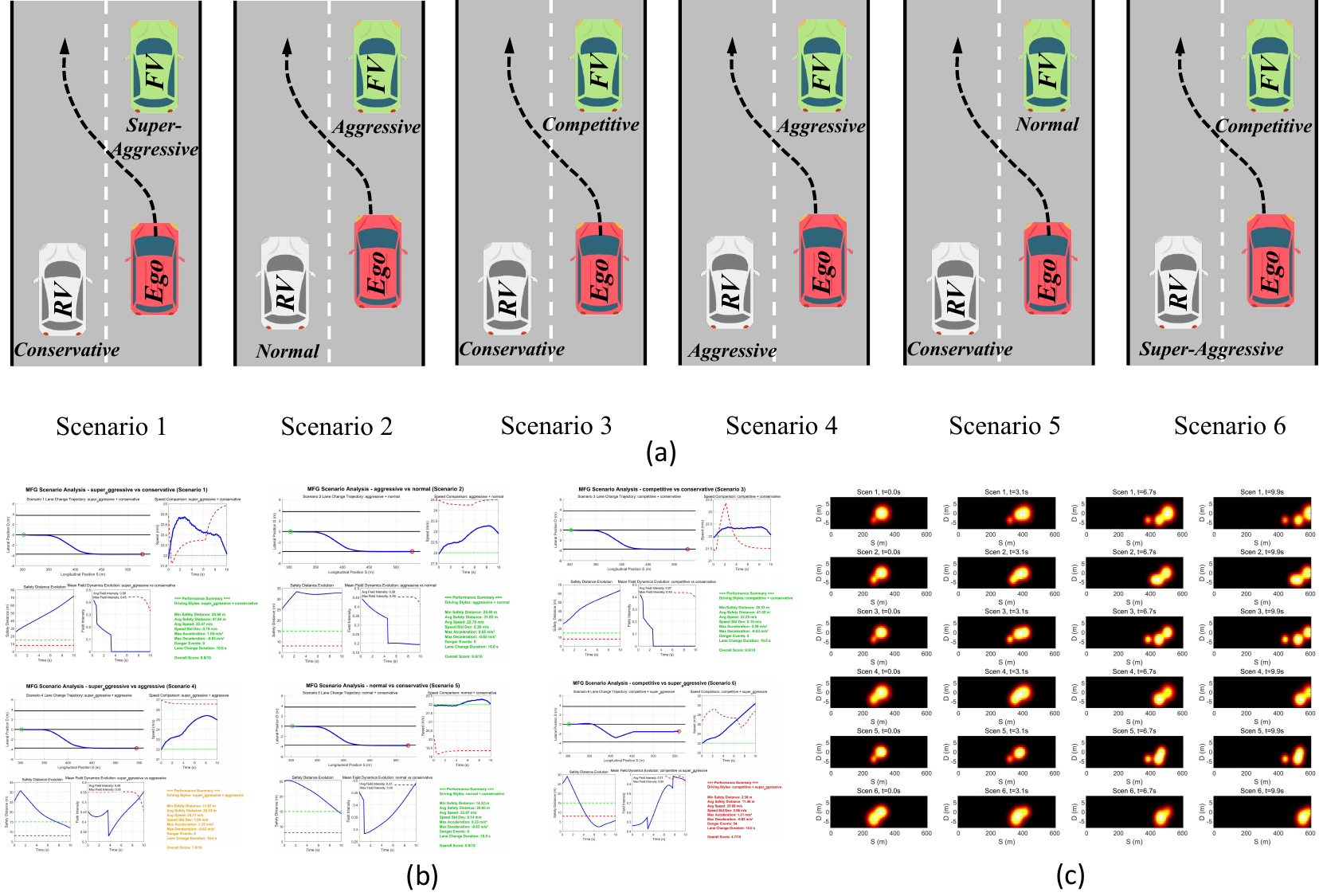}
    \caption{Mean Field Game analysis of lane change scenarios with different driving style combinations. (a) Schematic representation of six scenarios showing ego vehicle (red) performing lane changes while interacting with surrounding vehicles (green) of different driving styles. (b) Detailed performance analysis including trajectory plots, speed profiles, safety distances, and performance metrics for each scenario. (c) Mean field density evolution visualization showing spatial-temporal distribution patterns across four time points for each scenario.}
    \label{fig:mfg_scenarios}
\end{figure*}
\begin{theorem}[Uniqueness and Convergence Rate]
\label{thm:uniqueness_convergence}
Under the conditions of Theorem \ref{thm:existence} with the additional strong monotonicity condition:
\begin{equation}
\alpha_{\text{mono}} > \frac{L_f L_L}{2}
\end{equation}
where $L_f, L_L$ are the Lipschitz constants of $\boldsymbol{f}_i$ and $L_i$ respectively, the mean field equilibrium is unique. Moreover, the iterative algorithm:
\begin{equation}
\label{eq:iterative_algorithm}
\begin{aligned}
\boldsymbol{\mu}^{(k+1)} &= (1-\gamma)\boldsymbol{\mu}^{(k)} + \gamma \mathcal{T}[\{\mathbf{u}_i^{(k)}\}](\boldsymbol{\mu}^{(k)}) \\
\mathbf{u}_i^{(k+1)} &= (1-\gamma)\mathbf{u}_i^{(k)} + \gamma \mathcal{G}_i[\boldsymbol{\mu}^{(k+1)}](\mathbf{u}_i^{(k)})
\end{aligned}
\end{equation}
converges to the unique equilibrium with exponential rate:
\begin{equation}
\label{eq:convergence_rate}
W_2(\boldsymbol{\mu}^{(k)}, \boldsymbol{\mu}^*) + \sum_{i=1}^N \|\mathbf{u}_i^{(k)} - \mathbf{u}_i^*\|_{L^2} \leq \rho^k C_0
\end{equation}
where $\rho = (1-\gamma) + \gamma L_{\mathcal{T}} L_{\mathcal{G}} < 1$ for appropriately chosen $\gamma \in (0, 1)$, and $C_0$ depends on initial conditions.
\end{theorem}

%%%%%%%%%%%%%%%%%%%%%%%%%%%%%%%%%%%%%%%%%%%%%%%%%%%%%%%%%%%%%%%%%%%%%%%%%%%%%%%%%第七节

\section{Simulation Results}
\label{sec4}

The simulations were designed to verify the safety, stability and efficiency of the proposed method. The simulations were conducted on a computer with the Ubuntu 18.04.6 LTS OS, a 12th generation 16-thread Intel\textsuperscript{\textregistered}Core\texttrademark\ i5-12600KF CPU, an NVIDIA GeForce RTX 3070Ti GPU, and 16GB of RAM. The simulation results are obtained in MATLAB R2024b. 

To verify the effectiveness of the proposed ERPF-MPC, we designed several experimental scenarios. To show that the proposed MFG better assists the AV in safe driving under different driving-style combinations, we compared lane-change performance across six combinations of front-vehicle (FV, current lane) and rear-vehicle (RV, adjacent lane) styles. Additionally, we evaluated generalization using a lane-changing scenario based on the Next Generation Simulation (NGSIM) dataset to approximate real-world conditions. Finally, to highlight the superiority of ERPF in multi-driving-style interactive driving, we conducted lane-change simulations with 18 surrounding vehicles, each assigned its own driving style.
\begin{figure*}[t]
    \centering
    \includegraphics[width=0.8\textwidth]{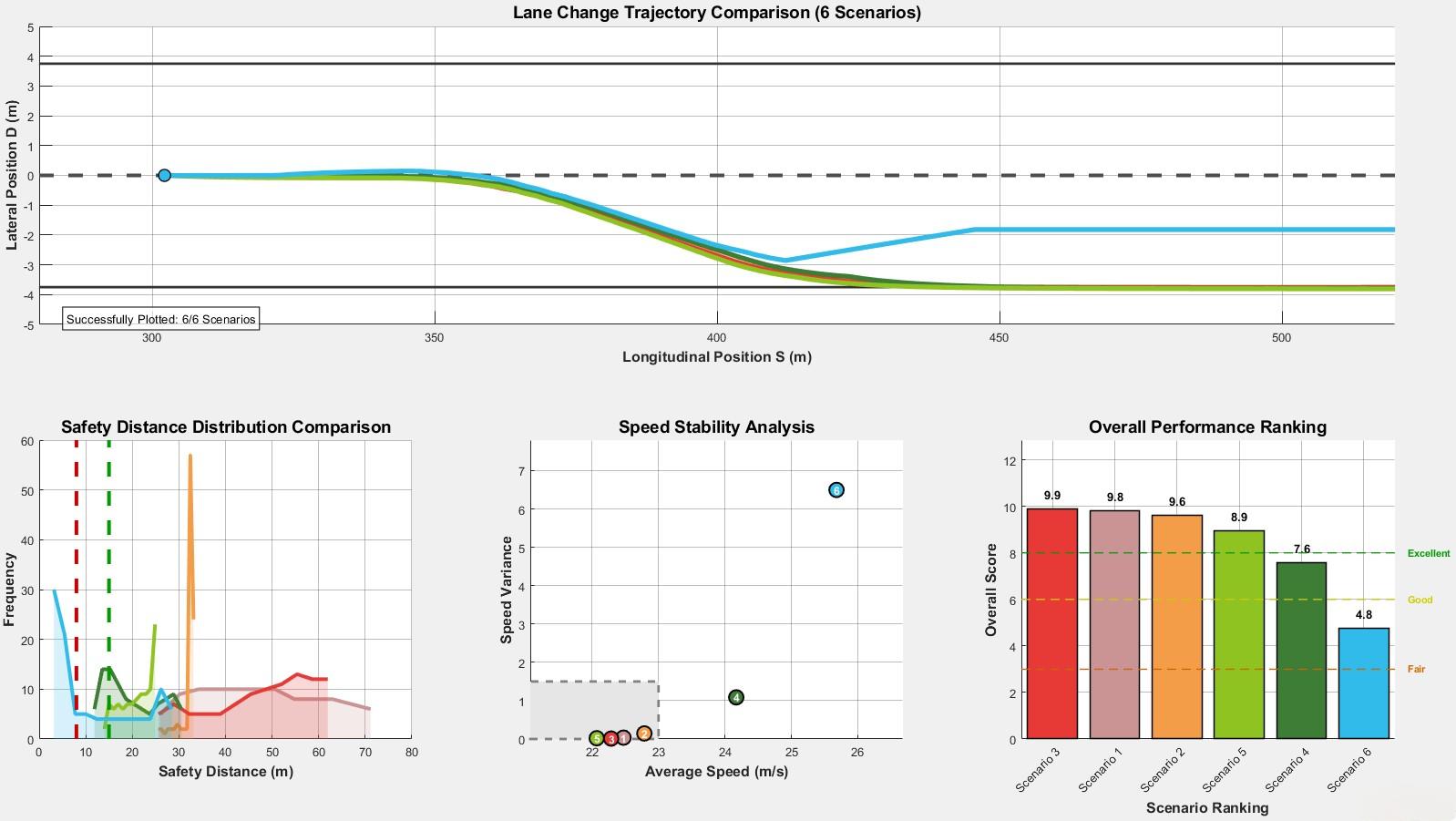}
    \caption{Comparative analysis of lane change performance across six driving style scenarios.}
    \label{fig:trajectory_comparison}
\end{figure*}

Figure~\ref{fig:mfg_scenarios} presents a comprehensive analysis of six different lane change scenarios in a MFG framework, examining the interaction between autonomous vehicles with varying driving styles during lane change maneuvers.

The experimental setup systematically evaluates six distinct combinations of driving styles: (1) Super-Aggressive vs Conservative, (2) Aggressive vs Normal, (3) Competitive vs Conservative, (4) Aggressive vs Aggressive, (5) Normal vs Conservative, and (6) Competitive vs Super-Aggressive. Each scenario demonstrates unique behavioral patterns and safety characteristics.

The results reveal significant variations in lane change performance across different driving style combinations. Panel (b) provides comprehensive performance metrics including trajectory plots showing the lateral displacement patterns, speed evolution curves indicating velocity adjustments during maneuvers, and safety distance monitoring throughout the lane change process. Scenarios involving aggressive drivers typically exhibit higher speed variations, more abrupt trajectory corrections, and reduced minimum safety distances, while conservative combinations produce smoother velocity profiles, gradual lateral movements, and consistently larger safety margins. The performance summary scores reflect these behavioral differences, with conservative scenarios achieving higher overall safety ratings.

Panel (c) presents the mean field density evolution across four temporal snapshots for each scenario, revealing distinct spatial-temporal patterns. The density fields exhibit varying concentration levels and distribution shapes depending on the driving style combinations. Aggressive scenarios demonstrate more localized, high-intensity density peaks that shift rapidly across the spatial domain, indicating dynamic vehicle clustering and frequent positional changes. Conservative scenarios show more diffuse, stable density distributions with gradual spatial transitions, reflecting predictable movement patterns and maintained inter-vehicle spacing. These density visualizations effectively capture the collective behavioral dynamics inherent in each driving style combination.

Figure~\ref{fig:trajectory_comparison} presents a comprehensive comparative analysis of lane change trajectories and performance metrics across six distinct driving style combinations, demonstrating the effectiveness of the Mean Field Game framework in capturing behavioral variations in autonomous vehicle interactions.

The trajectory comparison in the upper panel demonstrates a crucial finding: all six scenarios successfully completed collision-free lane change maneuvers from the middle lane (0 m) to the target lane (-3.75 m), validating the robustness of the Mean Field Game control framework across diverse driving style combinations. This achievement is particularly significant given the challenging nature of mixed-behavior traffic scenarios. Despite varying behavioral parameters, no collisions occurred in any scenario, indicating effective safety constraint enforcement within the MFG framework. The trajectories reveal distinct patterns while maintaining safety, with conservative driving styles exhibiting smoother, more gradual lateral transitions, while aggressive combinations demonstrate sharper trajectory changes with steeper descent rates, yet all remaining within safe operational bounds.

The safety distance distribution analysis in the lower left panel provides critical validation of the collision-free performance across all scenarios. While the frequency distributions show marked differences between scenarios, the fundamental safety integrity remains intact throughout all combinations. Conservative combinations (Scenarios 1, 3, and 5) maintain higher average safety distances with comfortable margins in the 15-40 meter range, while aggressive scenarios, despite operating with reduced safety margins and notable peaks in the 5-15 meter range, never breach the critical collision threshold. Importantly, even the most challenging scenarios maintain minimum distances well above zero, with the danger threshold at 8 meters serving as an effective safety boundary that prevents any collision events across all driving style interactions.

The speed stability analysis demonstrates the relationship between average velocity and speed variance, with the green ideal area representing optimal performance zones. Conservative scenarios cluster near the ideal region with lower variance and speeds closer to the target 22-23 m/s, while aggressive combinations show increased variance and higher average speeds, reflecting more dynamic velocity adjustments during lane changes. The overall performance ranking confirms the successful collision-free operation across all scenarios, with Scenario 3 (competitive+conservative) achieving the highest score of 9.9, followed closely by Scenarios 1 and 2. Notably, even the most aggressive combination (Scenario 6) maintains a score of 4.8 while achieving zero collisions, demonstrating that the MFG framework successfully balances assertive maneuvering with fundamental safety requirements. 

\subsection{NGSIM Real Data Integration Experiment}

This subsection introduces an enhanced Mean Field Game implementation that integrates real-world traffic data from the NGSIM (Next Generation Simulation) dataset to validate the framework's performance under authentic driving conditions. The experiment replaces synthetic surrounding vehicle behaviors with actual car-following data extracted from California highway traffic recordings, providing a more realistic testing environment for autonomous vehicle lane change scenarios.

\subsubsection{Experimental Setup and Data Preparation}
The NGSIM dataset provides rich behavioral information including longitudinal positions, velocities, accelerations, inter-vehicle spacing, relative velocities, time headways, and vehicle identification data. Each trajectory segment represents a continuous car-following sequence captured at ten-hertz frequency, delivering high-resolution temporal dynamics that reflect authentic human driving patterns observed on California highways.

\subsubsection{Coordinate System Transformation and Vehicle Positioning}
\begin{figure*}[t]
    \centering
    \includegraphics[width=\textwidth]{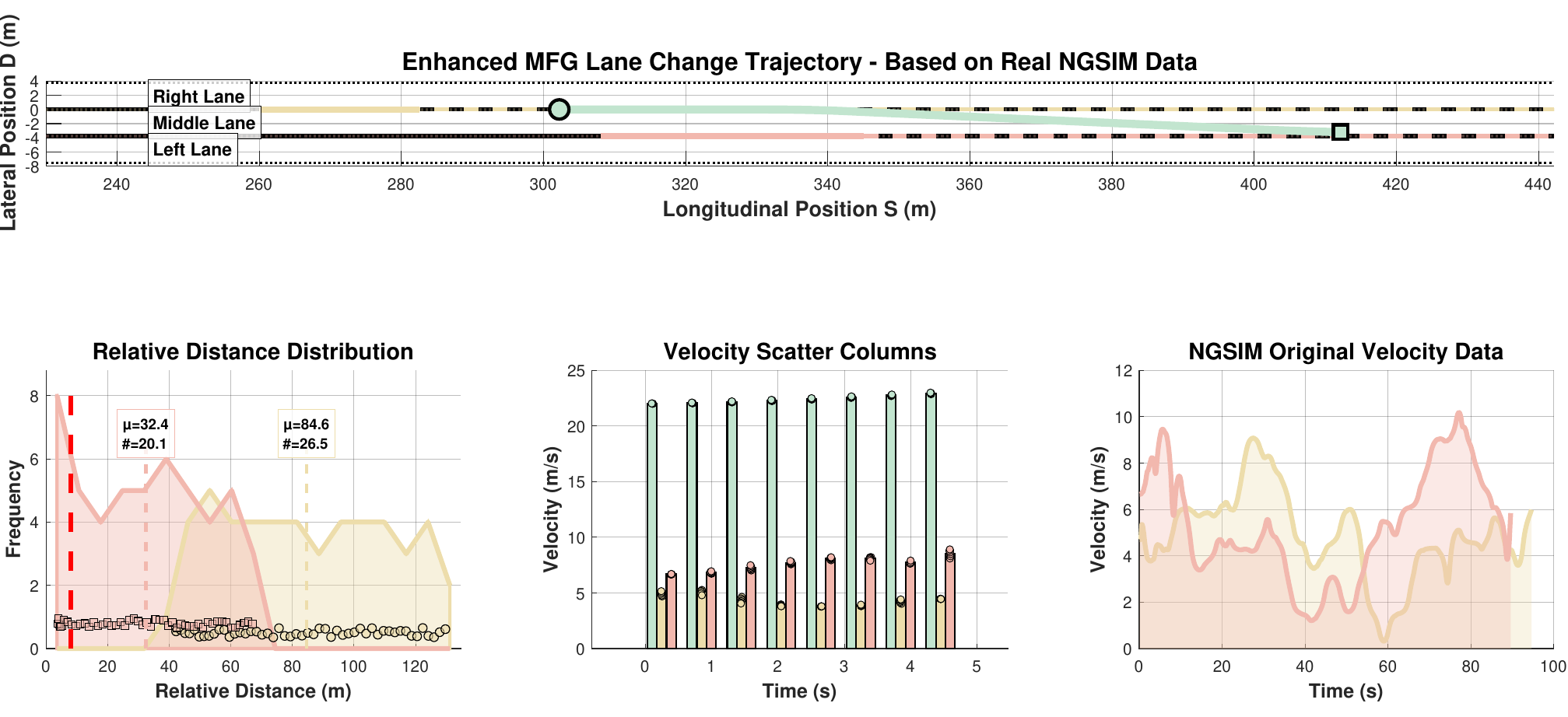}
    \caption{Comprehensive results analysis of real-world-data showing (top) trajectory evolution with real traffic data integration, (middle) safety metrics including relative distance and velocity distributions, and NGSIM velocity statistics.}
    \label{fig:ngsim_results}
\end{figure*}
\begin{figure*}[t]
  \centering
  % Replace the path below with your actual file name
  \includegraphics[width=\textwidth]{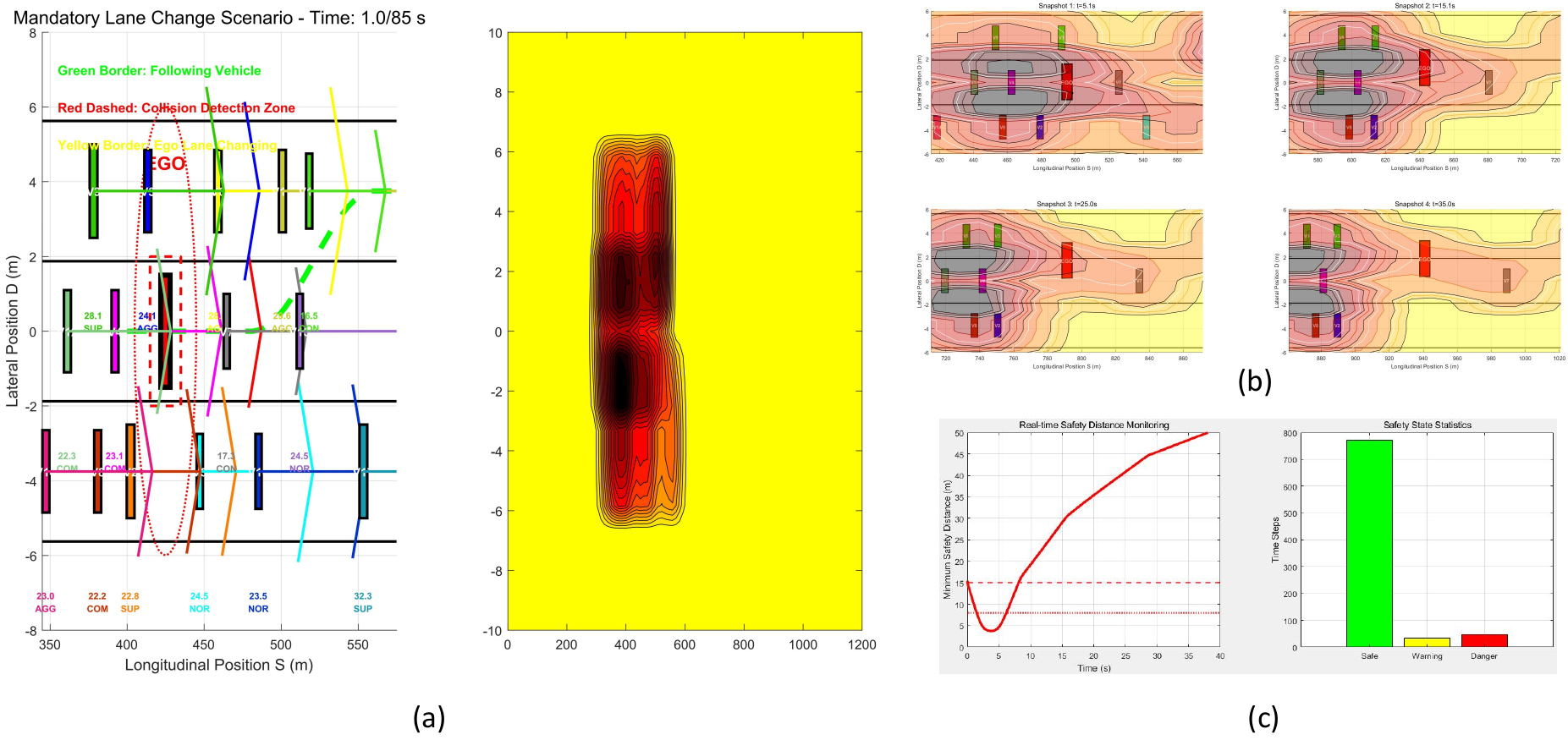}
  \caption{MFG planner under a hazardous configuration of Scenario 7. (a) Middle lane$\rightarrow$upper lane-change with tight headways. The mean-field density around the ego is highly concentrated at the start. (b) Four time-ordered snapshots where MFG coupling shapes a lower-density corridor in the target lane. (c) Safety metrics: the minimum inter-vehicle distance briefly approaches the danger threshold, then rises above the safety threshold and continues to increase.}
  \label{fig:results}
\end{figure*}
\begin{figure*}[t]
  \centering
  % replace with your file path
  \includegraphics[width=1\textwidth]{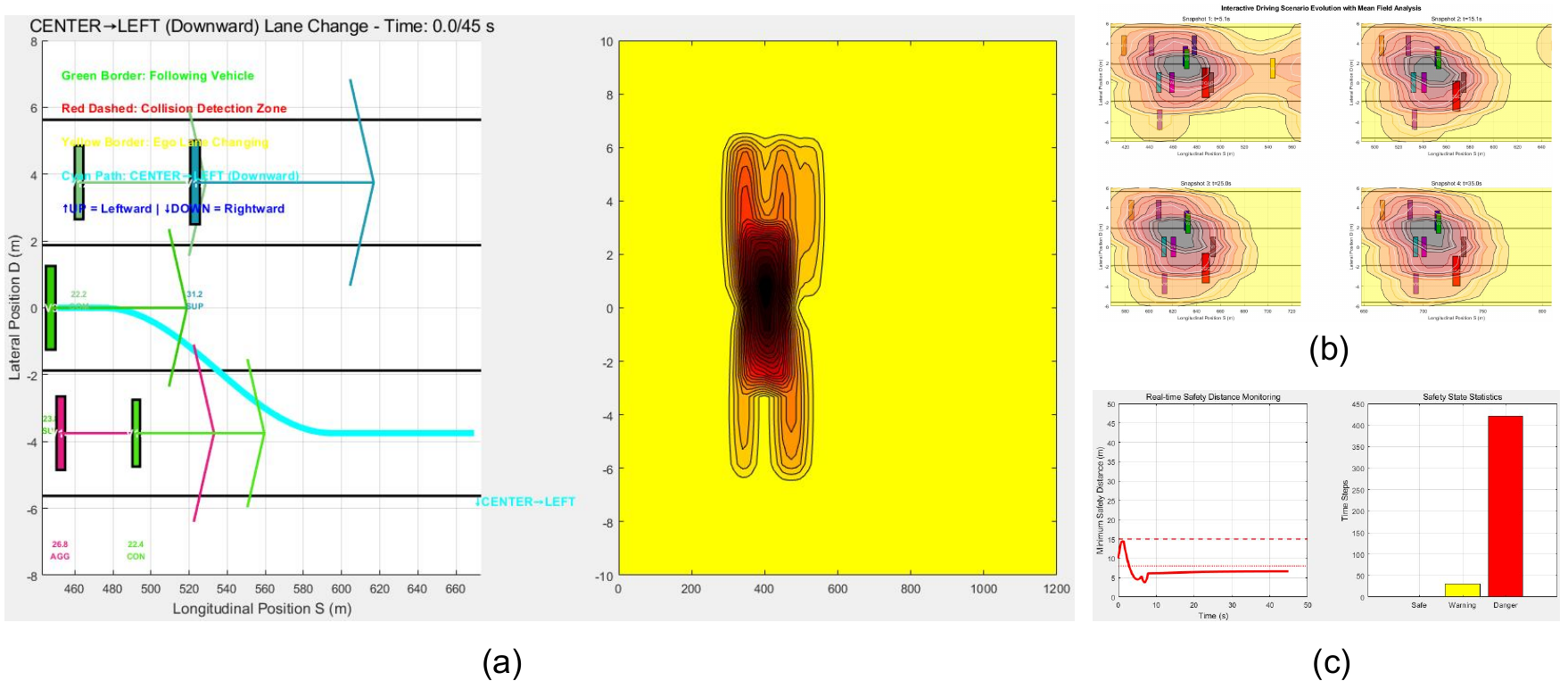}
  \caption{MFG planner under a hazardous configuration of Scenario 8. (a) Middle lane$\rightarrow$bottom lane-change with tight headways. (b) Four snapshots show how MFG interactions gradually reshape density and open a feasible corridor in the target lane. (c) Safety metrics: the minimum inter-vehicle distance dips into the danger regime and approaches the 8\,m threshold, yet remains strictly positive; the state histogram records many danger steps but no collisions.}
  \label{fig:mfg_danger}
\end{figure*}
In this simulation, we utilize real-world vehicle trajectory data from the NGSIM I-80 dataset to enhance the realism and validity of our MFG lane change analysis. The NGSIM dataset provides high-fidelity vehicular motion data collected from actual traffic conditions on Interstate 80 in Emeryville, California, with a sampling frequency of 10 Hz.

Specifically, we extract two representative car-following segments from the dataset to model the surrounding vehicles in our lane change scenario. The front vehicle (Vehicle ID: 2467) operates in the middle lane with 946 frames of trajectory data spanning approximately 94.6 seconds of driving behavior. The rear vehicle (Vehicle ID: 2155) is positioned in the left lane (target lane for ego vehicle's lane change maneuver) with 896 frames covering 89.6 seconds of naturalistic driving patterns.

The integration of authentic NGSIM data serves multiple purposes in our MFG framework: (1) it provides realistic velocity profiles and acceleration patterns that reflect actual human driving behavior, (2) it ensures that the surrounding vehicles exhibit naturalistic responses during the ego vehicle's lane change maneuver, and (3) it enables validation of our MFG-based decision-making against real-world traffic dynamics. The processed NGSIM trajectories are spatially and temporally aligned with our simulation environment while preserving the essential characteristics of human driving behavior, thereby creating a hybrid simulation environment that combines theoretical MFG modeling with empirical traffic data.

Figure~\ref{fig:ngsim_results} presents demonstrate the effectiveness of integrating real NGSIM data with theoretical mean field game modeling. The trajectory visualization reveals a smooth and safe lane change maneuver from the middle lane to the left lane, with the ego vehicle (light green) successfully navigating between the front vehicle (light yellow) and rear vehicle (light pink) while maintaining appropriate safety margins.

The relative distance distribution analysis provides crucial insights into the safety characteristics of the lane change maneuver, where $\mu$ represents the mean distance and $\#$ denotes the standard deviation. Two distinct distribution peaks are observed: the front vehicle distance distribution exhibits $\mu = 32.4$ meters with $\# = 20.1$ meters, indicating a relatively close but safe following distance with moderate variability. In contrast, the rear vehicle distance distribution shows $\mu = 84.6$ meters with $\# = 26.5$ meters, reflecting a more conservative gap maintenance strategy with slightly higher variance. Both distributions remain well above the critical 8-meter safety threshold, confirming the safety-conscious nature of the MFG-guided lane change decision.

The velocity scatter column analysis reveals consistent speed profiles across different time segments, with the ego vehicle maintaining speeds between 20-23 m/s while effectively coordinating with surrounding traffic. The NGSIM original velocity data demonstrates naturalistic driving patterns with realistic speed variations, ranging from 2-10 m/s, which validates the authenticity of the surrounding vehicle behaviors. The integration of these real-world velocity profiles ensures that the MFG framework operates under realistic traffic conditions, thereby enhancing the practical applicability of the proposed lane change algorithm.

\begin{table}[t]
\centering
\captionsetup{labelfont={sc}, textfont={sc}, labelsep=newline, justification=centering}
\renewcommand{\arraystretch}{1.1}
\caption{Complete Vehicle Configuration for Scenario 7: Middle-Upper Lane-changing during Multi-Vehicle Interaction}
\label{tab:vehicle_config}
\begin{threeparttable}
\setlength{\tabcolsep}{3.8pt}
\resizebox{\linewidth}{!}{%
\begin{tabular}{c|c|c|c|c|c}
\hline \hline
\multirow{2}{*}{Vehicle ID} & \multirow{2}{*}{Type} & \multirow{2}{*}{Driving Style} & \multirow{2}{*}{Lane} & Relative Position & Initial Speed \\
& & & & (m) & (m/s) \\
\hline
1  & Ego Vehicle & Autonomous       & Middle & 0     & 25.0 \\
\hline
2  & Vehicle 2   & Super Aggressive & Left   & -25   & 36.5 \\
\hline
3  & Vehicle 3   & Aggressive       & Right  & -15   & 32.4 \\
\hline
4  & Vehicle 4   & Competitive      & Middle & -35   & 27.4 \\
\hline
5  & Vehicle 5   & Normal           & Left   & +20   & 22.3 \\
\hline
6  & Vehicle 6   & Aggressive       & Right  & +30   & 32.5 \\
\hline
7  & Vehicle 7   & Conservative     & Middle & +45   & 15.2 \\
\hline
8  & Vehicle 8   & Competitive      & Left   & -45   & 29.7 \\
\hline
9  & Vehicle 9   & Super Aggressive & Right  & -55   & 33.9 \\
\hline
10 & Vehicle 10  & Normal           & Left   & +60   & 25.0 \\
\hline
11 & Vehicle 11  & Aggressive       & Right  & +70   & 33.5 \\
\hline
12 & Vehicle 12  & Competitive      & Middle & -65   & 30.9 \\
\hline
13 & Vehicle 13  & Normal           & Middle & +85   & 25.9 \\
\hline
14 & Vehicle 14  & Aggressive       & Left   & -80   & 30.8 \\
\hline
15 & Vehicle 15  & Conservative     & Right  & +100  & 17.4 \\
\hline
16 & Vehicle 16  & Competitive      & Right  & -90   & 27.3 \\
\hline
17 & Vehicle 17  & Super Aggressive & Left   & +120  & 34.6 \\
\hline
18 & Vehicle 18  & Normal           & Middle & -100  & 22.3 \\
\hline \hline
\end{tabular}%
}
\end{threeparttable}
\end{table}
\subsection{Multi-Vehicle Interaction Validation}

This subsection presents a comprehensive validation experiment designed to evaluate the Mean Field Game framework's performance in complex multi-vehicle environments. The experiment establishes a challenging scenario involving an autonomous ego vehicle navigating through dense traffic consisting of seventeen surrounding vehicles with diverse driving behaviors, creating a realistic and demanding test environment for autonomous vehicle control systems.

\subsubsection{Experimental Design and Vehicle Configuration}
\begin{table}[t]
\centering
\captionsetup{labelfont={sc}, textfont={sc}, labelsep=newline, justification=centering}
\renewcommand{\arraystretch}{1.1}
\caption{Complete Vehicle Configuration for Scenario 8: Middle-Bottom Lane-changing during Multi-Vehicle Interaction}
\label{tab:vehicle_config_full}
\begin{threeparttable}
\setlength{\tabcolsep}{4pt}
\resizebox{\linewidth}{!}{%
\begin{tabular}{c|c|c|c|c|c}
\hline \hline
\multirow{2}{*}{Vehicle ID} & \multirow{2}{*}{Type} & \multirow{2}{*}{Driving Style} & \multirow{2}{*}{Lane} & Relative Position & Initial Speed \\
& & & & (m) & (m/s) \\
\hline
1  & Ego Vehicle  & Autonomous       & Middle &   0   & 25.0 \\
\hline
2  & Vehicle 2    & Super Aggressive & Middle &  -20  & 24.0 \\
\hline
3  & Vehicle 3    & Aggressive       & Middle &  -10  & 20.8 \\
\hline
4  & Vehicle 4    & Competitive      & Middle &  -30  & 20.8 \\
\hline
5  & Vehicle 5    & Normal           & Middle &  +40  & 18.4 \\
\hline
6  & Vehicle 6    & Aggressive       & Middle &  +15  & 20.8 \\
\hline
7  & Vehicle 7    & Conservative     & Middle &  +25  & 14.4 \\
\hline
8  & Vehicle 8    & Competitive      & Middle &  +35  & 20.8 \\
\hline
9  & Vehicle 9    & Super Aggressive & Middle &  +45  & 24.0 \\
\hline
10 & Vehicle 10   & Normal           & Right  &  -35  & 21.9 \\
\hline
11 & Vehicle 11   & Aggressive       & Right  &  +30  & 24.7 \\
\hline
12 & Vehicle 12   & Competitive      & Right  &  +60  & 24.7 \\
\hline
13 & Vehicle 13   & Normal           & Right  &  -55  & 21.9 \\
\hline
14 & Vehicle 14   & Aggressive       & Left   &  +50  & 29.9 \\
\hline
15 & Vehicle 15   & Conservative     & Left   &  +90  & 20.7 \\
\hline
16 & Vehicle 16   & Competitive      & Left   &  -60  & 29.9 \\
\hline
17 & Vehicle 17   & Super Aggressive & Right  & +120  & 28.5 \\
\hline
18 & Vehicle 18   & Normal           & Right  &  -80  & 21.9 \\
\hline \hline
\end{tabular}%
}
\end{threeparttable}
\end{table}

The multi-vehicle validation experiment creates a densely populated traffic scenario on a three-lane highway section extending 1500 meters in length with standard lane widths of 3.75 meters. The ego vehicle, representing an autonomous vehicle equipped with the Mean Field Game control algorithm, is positioned at the center of this traffic configuration to maximize interaction complexity and provide comprehensive validation of the system's capabilities under realistic conditions.

The surrounding traffic environment consists of seventeen vehicles strategically distributed across all three lanes, with varying longitudinal positions ranging from 100 meters behind to 120 meters ahead of the ego vehicle. This configuration ensures continuous interaction throughout the simulation duration, creating conditions that closely mirror real highway traffic density and complexity. The vehicles are positioned to create both cooperative and competitive scenarios, testing the ego vehicle's ability to navigate through varying traffic patterns while maintaining safety and efficiency.

Table~\ref{tab:vehicle_config} presents the complete vehicle configuration, including driving styles, initial positions, lane assignments, and speed characteristics for all eighteen vehicles in the simulation. The configuration demonstrates a balanced distribution across vehicle types and positions, ensuring comprehensive testing of the Mean Field Game framework under diverse interaction scenarios.


% Analysis paragraph (plain text, no font changes)
Figure~\ref{fig:results} summarizes the MFG-based planner in scenario 7. Panel (a) shows the ego vehicle navigating among mixed driving styles while the mean-field density concentrates near the ego at the start and then redistributes. Panel (b) illustrates how MFG interactions shape a lower-density corridor in the target lane, producing a usable gap for a smooth merge. Panel (c) reports safety: the minimum distance briefly nears 8\,m, then exceeds 15\,m within the first seconds and keeps increasing; most steps are safe, with short warning periods and few danger events.

In the scenario 8, the ego vehicle is at higher risk because several surrounding vehicles start with shorter longitudinal gaps, which reduces the initial headway at \(t=0\). Consequently, the minimum separation can satisfy
\(d_{\min}(0) \approx d_{\text{danger}} (=8\,\mathrm{m}) < d_{\text{safe}} (=15\,\mathrm{m})\),
and the initial time-to-collision \(\mathrm{TTC}(0)\) becomes small. Being in the center lane also limits lateral escape space, while heterogeneous (aggressive/competitive) behaviors increase uncertainty. The MFG planner must first disperse local density to restore a feasible margin before executing a stable lane change.

Figure~\ref{fig:mfg_danger} evaluates the MFG planner in scenario 8 with short initial gaps and heterogeneous behaviors. In panel (a), the ego executes a center$\rightarrow$left lane change while the density field forms a compact plume around it, limiting early escape options. Panel (b) shows that the MFG coupling steers surrounding flow, creating a lower-density corridor that supports a controlled merge. Panel (c) confirms collision-free operation: although the minimum separation enters the danger band and nears the 8\,m threshold, we observe \(d_{\min}(t)>0\) and positive time-to-collision throughout the 45\,s horizon. Thus, even under persistent hazard exposure, the planner maintains nonzero clearance and completes the maneuver without collisions.

\section{Conclusion}  
This work proposed a heterogeneous MFG based planning framework for 
interactive autonomous driving. By formulating multi-agent dynamics with style-aware 
parametrization, heterogeneous interactions were compactly represented through a mean 
field approximation, where the aggregated intensity field provides a unified potential 
for behavior prediction and risk assessment. The forward--backward MFG solution yields 
an $\varepsilon$-Nash equilibrium policy, which is subsequently integrated into a 
safety-critical path planner that balances field consistency, safety margins, dynamic 
feasibility, and style adherence. Numerical evaluations and visualizations of the mean 
field intensity confirm the interpretability of the proposed modelling, while the 
integration with path planning demonstrates its capability to achieve adaptive, robust, 
and safe decision-making under complex traffic interactions.

\bibliographystyle{IEEEtran}
\bibliography{IEEEabrv,zq_lib}

\end{document}